\documentclass{article}

\PassOptionsToPackage{numbers, compress}{natbib}
 \usepackage[preprint]{neurips_2026}

\usepackage{natbib}
\setcitestyle{authoryear, open={(}, close={)}}
\usepackage{graphicx}
\usepackage{amsmath}
\usepackage[utf8]{inputenc} 
\usepackage[T1]{fontenc}    
\usepackage{hyperref}       

\usepackage{algorithm}
\usepackage{algpseudocode}
\usepackage{amsmath} 
\usepackage{url}            
\usepackage{booktabs}       
\usepackage{amsfonts}       
\usepackage{nicefrac}       
\usepackage{microtype}      
\usepackage{xcolor}         
\usepackage{amsthm}

\newtheorem{proposition}{Proposition}

\newtheorem{definition}{Definition}
\newtheorem{remark}{Remark}
\newtheorem{corollary}{Corollary}

\title{Bayesian Conformal Prediction via Decision-Theoretic Threshold Selection}

\author{%
  Fanyi Wu$^{1,2}$\thanks{Correspondence email: fanyi.wu@manchester.ac.uk.} \quad
  Veronika Lohmanova$^{1}$ \quad
  Samuel Kaski$^{1,3,4}$ \quad
  Michele Caprio$^{1}$ \\[0.4em]
  \small $^{1}$Department of Computer Science, University of Manchester, Manchester, UK \\
  \small $^{2}$UKRI AI Centre for Doctoral Training in Decision Making for Complex Systems \\
  \small $^{3}$Department of Computer Science, Aalto University, Espoo, Finland \\
  \small $^{4}$ELLIS Institute, Finland \\[0.2em]
}
\begin{document}

\maketitle


\begin{abstract}
We propose Bayesian Conformal Prediction (BCP), a framework that combines Bayesian posterior predictive distributions with PAC-style conformal risk control to produce prediction sets with finite-sample coverage guarantees. Standard quantile-threshold conformal methods often construct prediction sets using a single fixed threshold, which typically yields connected prediction sets. While valid, such sets can be inefficient when the posterior predictive distribution is multimodal, since they may span low-density regions between separated modes. The main contribution of BCP is to formulate conformal prediction as a decision-risk optimisation problem, extending standard fixed quantile-threshold sets to optimised highest posterior density (HPD) prediction sets. These sets can be disjoint, concentrating probability mass on separated high-density regions. Validity is enforced using a PAC-style risk constraint, which provides coverage control even when the Bayesian model is misspecified. In standard nested-threshold settings, BCP recovers the smallest feasible threshold, aligning with existing PAC-based approaches. In the multimodal experiment, HPD geometry substantially improves efficiency, reducing mean prediction set size from $4.82$ to $2.07$ while satisfying the target PAC pass rate. Across regression, classification, and distribution-shift experiments, BCP maintains reliable coverage under model misspecification, whereas Bayesian credible intervals can fail to preserve nominal coverage.
\end{abstract}
\section{Introduction}
Uncertainty quantification is essential for reliable machine learning.
Bayesian inference, deep ensembles, and calibration methods have improved
uncertainty estimation
\citep{angelopoulos2021gentle, gal2016dropout, tyralis2024review,
caprio2024credal}, but their reliability can degrade when the underlying probabilistic model is misspecified.

\emph{Conformal Prediction} (CP) offers a complementary route, providing
finite-sample, distribution-free coverage under exchangeability
\citep{vovk2005algorithmic, shafer2008tutorial, barber2023conformal,
caprio2025conformalized}. Standard split CP defines a non-conformity score
$s(x,y)$ and selects a threshold $\lambda$ by quantile calibration, producing sets of
the form
\begin{equation}
C_{\mathrm{cp}}(x_{n+1})
=
\left\{
y \in \mathcal{Y} \colon s(x_{n+1}, y) \le \lambda_{\mathrm{cp}}
\right\}.
\end{equation}
This guarantees validity, but fixes a single threshold from calibration data. Under PAC-style risk control, however, coverage needs only to hold with high probability over the randomness of calibration. Hence, by inducing a \emph{feasible set} of valid thresholds $\lambda$, it is possible to optimise within this set rather than accept the quantile threshold as fixed.

This perspective raises a natural question: \emph{can conformal threshold $\lambda$ be optimised using Bayesian posterior predictive distributions, and obtain geometrically adaptive prediction sets under multimodal distributions?}

We propose \textbf{Bayesian Conformal Prediction (BCP)} to address this question. BCP casts CP as a PAC-constrained decision-risk
optimisation problem: the conformal threshold $\lambda$ is treated as a
decision variable and selected to minimise expected prediction set size while preserving coverage guarantees. In standard threshold settings, this recovers the smallest feasible threshold $\lambda^*$ used by existing PAC-based methods, whereas under multimodal posterior predictive distributions, it enables geometrically adaptive HPD conformal sets. Our contributions include:

\textbf{Posterior-predictive non-conformity scores.}
BCP uses non-conformity scores
$s(x,y) = -\log \widehat{p}(y \mid x, D_{\mathrm{tr}})$ to incorporate parameter and predictive uncertainty while preserving exchangeability. The posterior predictive densities used to compute these scores are approximated via add-one-in (AOI) importance sampling (Appendix~\ref{Appendix:AOI})

\textbf{HPD conformal sets with adaptive geometry.}
Under multimodal predictive distributions, HPD sets can be disjoint, unlike
interval-based residual Split-CP and CQR. In our multimodal experiment, this
reduces the mean set size from $4.82$ to $2.07$
(Table~\ref{tab:multimodal}, $p<0.0001$).

\textbf{PAC-controlled threshold selection.}
Coverage is enforced using the $L^+$ statistic from conformal risk control
\citep{angelopoulos2025conformalriskcontrol}. In HPD settings, Bayesian
quadrature stabilises objective estimation near mode-boundary density levels
(Section~\ref{sec:risk}, Appendix~\ref{app:hpd-bq}).

\textbf{Robustness under misspecification.}
Because coverage relies on exchangeability rather than model correctness, BCP restores near-nominal coverage ($80.0\%$) where Bayesian credible intervals collapse to $49.2\%$ under prior misspecification (Table~\ref{tab:regression_all}).

\section{Related Work}
\label{sec:relatedwork}
\textbf{Bayesian and conformal prediction.} CP provides finite-sample coverage under exchangeability
\citep{vovk2005algorithmic, shafer2008tutorial}. In contrast, Bayesian methods offer
posterior uncertainty, but can degrade when the model is misspecified \citep{bernardo2009bayesian, jansen2013robust}. Several works
combine these two views. Conformal Bayesian Computation \citep{fong2021conformalbayesiancomputation} uses posterior predictive scores
within full CP. Conformalized Bayesian Inference
\citep{bariletto2025conformalized} instead constructs uncertainty regions in parameter space. Our proposal takes a different direction. It operates in the output space under split CP and shows how the posterior predictive density can shape the prediction set itself, especially under multimodal distributions.

\textbf{Efficiency of conformal prediction.}
Much of the work on CP efficiency improves the non-conformity score. This
includes asymptotically optimal scores \citep{sadinle2018least}, learned scores \citep{bellotti2021optimized, stutz2022learning}, and Conformalized Quantile Regression (CQR) \citep{romano2019conformalized}. These methods improve the score, whereas the threshold $\lambda$ is chosen by empirical quantile calibration. In regression, this usually produces connected intervals. BCP changes both the score and the threshold. It uses a posterior-predictive non-conformity score and treats the threshold $\lambda$ as a decision variable. The optimal threshold $\lambda^*$ is selected to minimise expected set size subject to a PAC-style coverage constraint. This allows HPD conformal sets to be disjoint under multimodal posteriors, avoiding low-density regions that interval-based methods include.

\textbf{Decision-theoretic calibration and PAC guarantees.}
\citet{snell2025conformal} interprets split CP through Bayesian quadrature (BQ) and uses the full posterior distribution over the expected loss for PAC-style threshold selection. Their focus is on failure-rate control; prediction-set geometry and multimodal
settings are not explicitly addressed. BCP also adopts PAC coverage control but
targets a different problem: constructing efficient HPD prediction sets whose
geometry adapts to the predictive distribution. In nested-threshold settings, both approaches recover the smallest feasible $\lambda^*$. Differences arise under multimodal posteriors, where the HPD set geometry changes non-smoothly as
the density level crosses a mode boundary, increasing estimator variance and destabilising threshold selection. BCP uses BQ-based objective estimation to
stabilise this step and extend PAC threshold selection to disjoint prediction sets. 
\section{CP as a Decision-Risk Problem}
\label{sec:risk}
We reformulate CP as a decision-risk optimisation problem in which the goal is to construct the smallest prediction sets subject to a PAC-style coverage constraint, making explicit the trade-off between coverage guarantees and set size.

Formally, we parametrise CP sets by a real-valued threshold
$\lambda \in \mathbb{R}$, writing $C(x;\lambda)$ for the induced prediction
set at input $x$.
In classical split CP, the threshold is fixed by the empirical quantile calibration,
\begin{equation}
\lambda_{\mathrm{cp}} = s_{(\lceil (n+1)(1-\alpha)\rceil)},
\label{eq:split-cp-threshold}
\end{equation}
where $s_{(1)} \leq \cdots \leq s_{(n)}$ are the ordered non-conformity scores.
BCP instead treats $\lambda$ as a decision variable and enforces coverage with high probability over calibration randomness:
\begin{equation}
\min_{\lambda}\ \mathbb{E}_{X}\!\left[\,|C(X;\lambda)|\,\right]
\quad \text{s.t.}\quad
\mathbb{P}_{\mathcal{D}}\!\left(
  \mathbb{P}_{(X,Y)}\!\left(Y \notin C(X;\lambda)\right) \leq \alpha
\right) \geq 1-\beta,
\label{eq:decision-risk-main}
\end{equation}
where $\lambda = \hat{\lambda}(\mathcal{D})$ depends on calibration data $\mathcal{D}$, and $\beta \in (0,1)$ is the failure probability.
This formulation reveals when conformal prediction reduces to a trivial threshold-selection problem and when it becomes a genuinely non-trivial optimisation.

\paragraph{Structure of the optimisation.}

In the threshold-based setting $C(x;\lambda) = \{y : s(x,y) \leq \lambda\}$, the objective $\mathbb{E}_{X}[|C(X;\lambda)|]$ is monotone non-decreasing in $\lambda$, and the feasible set is a half-line $[\lambda_{\min}, \infty)$ (Fig.~\ref{fig:threshold_cp_demo}).
The optimisation, therefore, reduces to selecting the smallest feasible $\lambda$, of which \citet{snell2025conformal} is a special case.

\begin{proposition}
\label{prop:monotone}
Let $s(x,y)$ be fixed with respect to calibration labels and define
$C(x;\lambda) = \{y : s(x,y) \leq \lambda\}$.
Then $\mathbb{E}_{X}[|C(X;\lambda)|]$ is monotone non-decreasing in
$\lambda$, the feasible region of~\eqref{eq:decision-risk-main} is a
half-line $[\lambda_{\min}, \infty)$, and the unique minimiser is
$\lambda^* = \lambda_{\min}$.
\end{proposition}

A formal proof appears in Appendix~\ref{app:monotonicity-hpd}.
This contrasts with the HPD setting, where abrupt geometry changes near mode-boundary density levels make stable identification of $\lambda^*$ non-trivial.

This structure becomes more delicate in the HPD setting.
Although $\mathbb{E}_{X}[|C_{\mathrm{HPD}}(X;\lambda)|]$ remains monotone non-increasing in $\lambda$, the geometry of HPD sets changes abruptly as $\lambda$ crosses a mode-boundary density level.
At these points, a connected component of the set disappears, leading to high estimator variance with finite posterior samples.
This instability makes reliable identification of $\lambda^*$ difficult near such thresholds, motivating the use of Bayesian quadrature for objective estimation (Section~\ref{sec:multimodal}).

\paragraph{Role of Bayesian quadrature.}

We approximate $g(\lambda) = \mathbb{E}_{X}[|C(X;\lambda)|]$ using Bayesian
quadrature (BQ), which places a Gaussian process prior over
$\lambda \mapsto g(\lambda)$ and shares information across candidate values.
BQ is used solely as an estimation tool and does not affect the validity
constraint. It is most beneficial when $g(\lambda)$ is costly or unstable to estimate in regression, a task that involves integrating over a continuous label space.
In HPD settings, estimator variance increases sharply near the mode-boundary
thresholds (Appendix~\ref{app:hpd-bq}).
In multiclass classification, where scores are pre-computable, dense evaluation is inexpensive, and BQ can be optional.

\paragraph{Bayesian non-conformity scores.}

BCP constructs non-conformity scores from the posterior predictive density,
\begin{equation}
s(x,y) = -\log \widehat{p}(y \mid x, D_{\mathrm{tr}}),
\label{eq:bayes-score}
\end{equation}
incorporating both parameter and predictive uncertainty.
The score depends only on the training posterior and is fixed with respect to calibration labels, preserving exchangeability. Coverage guarantees
follow from conformal risk
control~\citep{angelopoulos2025conformalriskcontrol}
(Appendix~\ref{app:risk-theory}).

Thresholding $s(x,y) = -\log\hat{p}(y\mid x)$ at level $\lambda$ is equivalent to the HPD superlevel set
$\{y : \hat{p}(y\mid x) \geq e^{-\lambda}\}$,
which is itself disjoint under multimodal posteriors.
The efficiency advantage of BCP in such settings arises from this geometry.
In contrast, interval-constrained methods such as residual Split-CP and CQR enforce connected prediction sets and must span the low-density valley between modes (Section~\ref{sec:multimodal}).
This highlights that efficiency gains in BCP are fundamentally driven by the geometry of the predictive distribution rather than threshold selection alone.

\section{Method}
\label{sec:method}
We implement the decision-risk formulation in~\eqref{eq:decision-risk-main} using three components: a Bayesian posterior-aware non-conformity score, BQ for efficiency estimation, and Dirichlet-based risk control
via $L^+$ to ensure PAC coverage.

\subsection{AOI Posterior Sampling}
\label{subsec:AOI}
Given $\mathcal{D}_{\mathrm{tr}} = \{(X_i,Y_i)\}_{i=1}^n$, let
$p(\theta \mid \mathcal{D}_{\mathrm{tr}})$ denote the posterior over parameters,
and let $f_\theta(y \mid x)$ be the conditional model. Given posterior samples
$\{\theta^{(t)}\}_{t=1}^T \sim p(\theta \mid \mathcal{D}_{\mathrm{tr}})$, we approximate posterior predictive densities using add-one-in (AOI) importance sampling \citep{fong2021conformalbayesiancomputation}. AOI reweights existing posterior samples to approximate the predictive density at arbitrary $(x, y)$ without refitting the model. This avoids the cost of running a separate posterior inference for each calibration point, making score 
computation substantially more efficient than full posterior refitting.
For a candidate label $y$ at input $x$, the AOI weights are
\begin{equation}
  \tilde{w}^{(t)}
  = \frac{f_{\theta^{(t)}}(y \mid x)}
         {\sum_{t'=1}^T f_{\theta^{(t')}}(y \mid x)},
  \label{eq:AOI-weight}
\end{equation}
and the posterior predictive density is approximated by
\begin{equation}
  \widehat{p}(y \mid x, \mathcal{D}_{\mathrm{tr}})
  = \sum_{t=1}^T \tilde{w}^{(t)}\, f_{\theta^{(t)}}(y \mid x).
\end{equation}
This yields the non-conformity score
\begin{equation}
  s(x,y) = -\log \widehat{p}(y \mid x, \mathcal{D}_{\mathrm{tr}}).
  \label{eq:score}
\end{equation}

As a result, the non-conformity score~\eqref{eq:score} is fixed with respect to the calibration data, satisfying the requirement of the exchangeability of CP.

\subsection{Bayesian Quadrature for Efficiency Estimation}
\label{subsec:bq}
As established in Section~\ref{sec:risk}, our goal is to identify the smallest feasible $\lambda$ while satisfying the PAC coverage constraint. The efficiency objective in \eqref{eq:decision-risk-main}:
\begin{equation}
  g(\lambda) 
  = \mathbb{E}_X\!\left[|C(X;\lambda)|\right]
  = \int |C(x;\lambda)|\,p(x)\,dx
  \label{eq:bq-target}
\end{equation}
is generally intractable. We approximate it using Bayesian quadrature (BQ) \citep{ohagan1991bayes}, which places a Gaussian process prior over $\lambda \mapsto g(\lambda)$ and shares information across candidate values.

BQ is used to estimate $g(\lambda)$; the $L^+$ constraint determines feasibility and coverage. It is particularly important in the HPD setting, when estimator variance increases sharply near mode-boundary density levels, making stable threshold selection non-trivial.


\subsection{Empirical Risk Control via $L^+$}
\label{subsec:lplus}

We adopt the CRC framework of \citet{snell2025conformal}.
Let $\ell_i(\lambda) = \mathbf{1}\{Y_i \notin C(X_i;\lambda)\}$ denote the
miscoverage loss, and introduce a virtual loss $\ell_{n+1} \equiv B$,
where $B$ is a known upper bound on the loss. The $L^+$ statistic stochastically upper-bounds the true miscoverage risk \citep{angelopoulos2025conformalriskcontrol}, enabling reliable threshold
selection under finite samples:
\begin{equation}
  L^+(\lambda)
  = \sum_{i=1}^{n+1} U_i\, \ell_{(i)}(\lambda),
  \qquad U \sim \mathrm{Dir}(1,\dots,1),
  \label{eq:lplus}
\end{equation}
where $\ell_{(1)} \leq \cdots \leq \ell_{(n+1)}$ are the ordered losses,
with $\ell_{(n+1)} = B$.
The threshold is selected as
\begin{equation}
  \lambda^* = \min\!\bigl\{\lambda : \mathbb{P}_U[L^+(\lambda) \leq \alpha]
              \geq 1-\beta\bigr\},
  \label{eq:threshold}
\end{equation}
which ensures coverage $1-\alpha$ with probability at least $1-\beta$.
For threshold-based prediction sets (monotone in $\lambda$), the feasible
region is $[\lambda_{\min}, \infty)$ and the smallest feasible $\lambda$ is selected.
For HPD sets (monotone non-increasing in $\lambda$), the feasible region
becomes $[0, \lambda_{\max}]$ and the largest feasible $\lambda$ is selected
(Appendix~\ref{app:hpd}).

\paragraph{Summary.}
Posterior samples capture parameter uncertainty, AOI defines a fixed
non-conformity score, and $L^+$ enforces PAC-style coverage.
The selected $\lambda^*$ is the boundary of the feasible region satisfying the risk constraint, while BQ provides a stable estimate of efficiency.
Algorithm~\ref{alg:bcp-main} summarises the procedure.
In multimodal settings, the HPD variant follows the same pipeline with a density-based threshold; details are provided in Appendix~\ref{app:alg:hpd}.

\begin{algorithm}[t]
\caption{BCP (AOI + BQ + $L^+$)}
\label{alg:bcp-main}
\begin{algorithmic}[1]
\Require Training data $\mathcal{D}_{\mathrm{tr}}$, calibration data
  $\mathcal{D}_{\mathrm{cal}} = \{(X_i,Y_i)\}_{i=1}^n$,
  test input $x_{n+1}$, miscoverage level $\alpha$,
  failure probability $\beta$
\Ensure Prediction set $C(x_{n+1};\lambda^*)$
\State Draw $\{\theta^{(t)}\}_{t=1}^T \sim p(\theta \mid
  \mathcal{D}_{\mathrm{tr}})$
  \hfill\Comment{posterior fit on training data only}
\State Compute calibration scores $s_i = s(X_i,Y_i)$
  via AOI~\eqref{eq:AOI-weight}--\eqref{eq:score}
\State Estimate $g(\lambda)$ via BQ over a candidate $\lambda$ grid
\State Compute losses $\ell_i(\lambda)$ and evaluate
  $L^+$ constraint~\eqref{eq:lplus} for each $\lambda$
\State Set $\lambda^* \leftarrow \min\{\lambda :
  \mathbb{P}_{\mathcal{D}}[L^+(\lambda) \leq \alpha] \geq 1-\beta\}$
\State Compute test scores $s(x_{n+1},y)$ for $y$ in a label grid
\State \Return $C(x_{n+1};\lambda^*) =
  \{y : s(x_{n+1},y) \leq \lambda^*\}$
\end{algorithmic}
\end{algorithm}
\section{Experiments}
\label{sec:experiment}
We evaluate BCP on regression and classification tasks against the standard
baselines: Split-CP \citep{lei2018distributionfree}, Conformal Bayesian Computation (CB) \citep{fong2021conformalbayesiancomputation}, Bayesian Credible
Intervals (BCI) \citep{gelman2013bayesian}, Conformalised Quantile Regression (CQR) \citep{romano2019conformalized}, and Snell-HPD \citep{snell2025conformal}. BCP prioritises coverage guarantee and stability under model misspecification.
Rather than minimising prediction-set size alone, it enforces a PAC-style coverage constraint, which may not yield the smallest set size but reduces variability in set size and ensures empirical coverage across repeated data splits. All experiments use Python with
\texttt{PyTorch}, \texttt{Pyro}, and \texttt{scikit-learn}.
Implementation details are deferred to Appendix~\ref{app:exp-details}.
\subsection{Sparse Regression on Diabetes}
\label{sec:diabetes}
We use the diabetes
dataset of \citet{efron2004least}: $n{=}442$ samples, $d{=}10$ features, and a
continuous disease-progression response.
The likelihood is Gaussian,
\begin{equation}
    f_\theta(y \mid x) = \mathcal{N}(y \mid \theta^\top x + \theta_0,\,\tau^2),
\end{equation}
with hierarchical priors $\pi(\theta_j) = \mathrm{Laplace}(0,b)$,
$\pi(b) = \mathrm{Gamma}(1,1)$, $\pi(\tau) = \mathcal{N}^+(0,c)$,
where $c\!\in\!\{1.0,\,0.02\}$ governs prior concentration.
Standardisation is applied after splitting to prevent leakage.
Both settings use MCMC with $T{=}8000$ samples; $c{=}0.02$ corresponds to a
misspecified prior \citep{jansen2013robust}.
Baselines include Split-CP, BCI, CB,
CQR, Snell-HPD, and BCP.
We fix $\alpha{=}0.2$ (target coverage: $1-\alpha = 80\%$), $\beta{=}0.2$ (confidence level: $1-\beta = 80\%$) and evaluate over 50 random splits
(52.5\% train / 17.5\% calibration / 30\% test).

\begin{table}[t]
\centering
\caption{Regression results on the Diabetes dataset (target coverage: $80\%$;
$\alpha{=}0.2$, $\beta{=}0.2$; 50 random splits).
\textcolor{red}{Red} marks serious under-coverage; \textbf{bold} marks BCP. BCI provide conditional coverage under the correct prior,
but offers no guarantee under model misspecification.}
\label{tab:regression_all}
\begin{tabular}{lcccccc}
\toprule
\textbf{Method} & \textbf{Guarantee} & $\boldsymbol{c}$
  & \textbf{Cov.\,(\%)} & \textbf{Width}
  & \textbf{Pred.\,(s)} & \textbf{Calib.\,(s)} \\
\midrule
Split-CP & marginal & $1.0$  & $78.7\pm5.5$ & $1.81\pm0.19$ & $<10^{-4}$ & $0.002$ \\
Split-CP & marginal & $0.02$ & $78.8\pm5.5$ & $1.81\pm0.18$ & $<10^{-4}$ & $0.002$ \\
\midrule
BCI & conditional & $1.0$  & $78.5\pm3.9$ & $1.77\pm0.05$ & $1.0\!\times\!10^{-3}$ & --- \\
BCI & none & $0.02$ & \textcolor{red}{$49.2\pm4.3$} & $1.00\pm0.02$
    & $1.0\!\times\!10^{-3}$ & --- \\
\midrule
CQR & marginal & $1.0$  & $80.9\pm6.0$ & $1.90\pm0.21$ & $1.0\!\times\!10^{-3}$ & $0.101$ \\
CQR & marginal & $0.02$ & $80.8\pm5.3$ & $1.90\pm0.20$ & $1.0\!\times\!10^{-3}$ & $0.086$ \\
\midrule
CB & marginal & $1.0$  & $79.2\pm4.8$ & $1.80\pm0.08$ & $3.0\!\times\!10^{-3}$ & $7.808$ \\
CB & marginal & $0.02$ & $79.2\pm4.4$ & $1.81\pm0.07$ & $3.0\!\times\!10^{-3}$ & $7.437$ \\
\midrule
Snell-HPD & PAC & $1.0$  & $79.3\pm5.5$ & $1.83\pm0.19$ & $1.0\!\times\!10^{-3}$ & $0.105$ \\
Snell-HPD & PAC & $0.02$ & $79.3\pm5.6$ & $1.82\pm0.18$ & $1.0\!\times\!10^{-3}$ & $0.099$ \\
\midrule
\textbf{BCP} & \textbf{PAC} & $\boldsymbol{1.0}$
  & $\boldsymbol{80.1\pm5.8}$ & $\boldsymbol{1.87\pm0.19}$
  & $\boldsymbol{2.0\!\times\!10^{-3}}$ & $\boldsymbol{1.152}$ \\
\textbf{BCP} & \textbf{PAC} & $\boldsymbol{0.02}$
  & $\boldsymbol{80.0\pm5.5}$ & $\boldsymbol{1.87\pm0.18}$
  & $\boldsymbol{2.0\!\times\!10^{-3}}$ & $\boldsymbol{1.225}$ \\
\bottomrule
\end{tabular}
\end{table}

Table~\ref{tab:regression_all} and Figure~\ref{fig:reg_comparison} present
the results. Three findings deserve attention.

\textit{Isolating efficiency optimisation.}
BCP and Snell-HPD share identical $L^+$ enforcement on the coverage constraint; the difference is that BCP explicitly minimises $\mathbb{E}_x[|C(x;\lambda)|]$
while Snell-HPD selects the smallest feasible $\lambda$ without a further
efficiency objective.
BCP ensures the target coverage ($80.1\%$ vs.\ $79.3\%$) at the cost of a slightly wider interval ($1.87$ vs.\ $1.83$) and higher calibration cost
($1.15$\,s vs.\ $0.10$\,s). The overhead arises from the BQ step.

\textit{Robustness to prior misspecification.} 
BCI collapses to $49.2\%$ coverage under $c{=}0.02$ despite producing the narrowest intervals. BCP restores near-nominal coverage ($80.0\%$) in the
same setting, demonstrating that conformal calibration via $L^+$ corrects for Bayesian miscalibration independently of prior quality. All marginal-coverage methods hold near $79$--$81\%$ across prior settings, confirming that
exchangeability drives their validity rather than the model.

\textit{Trade-off between split and full CP.} Although CB gives lower variability across run splits, it is a full CP method that requires much more computational cost than BCP ($7.81$ s vs. $1.15$ s).

\begin{figure}[h]
\centering
\includegraphics[width=\linewidth]{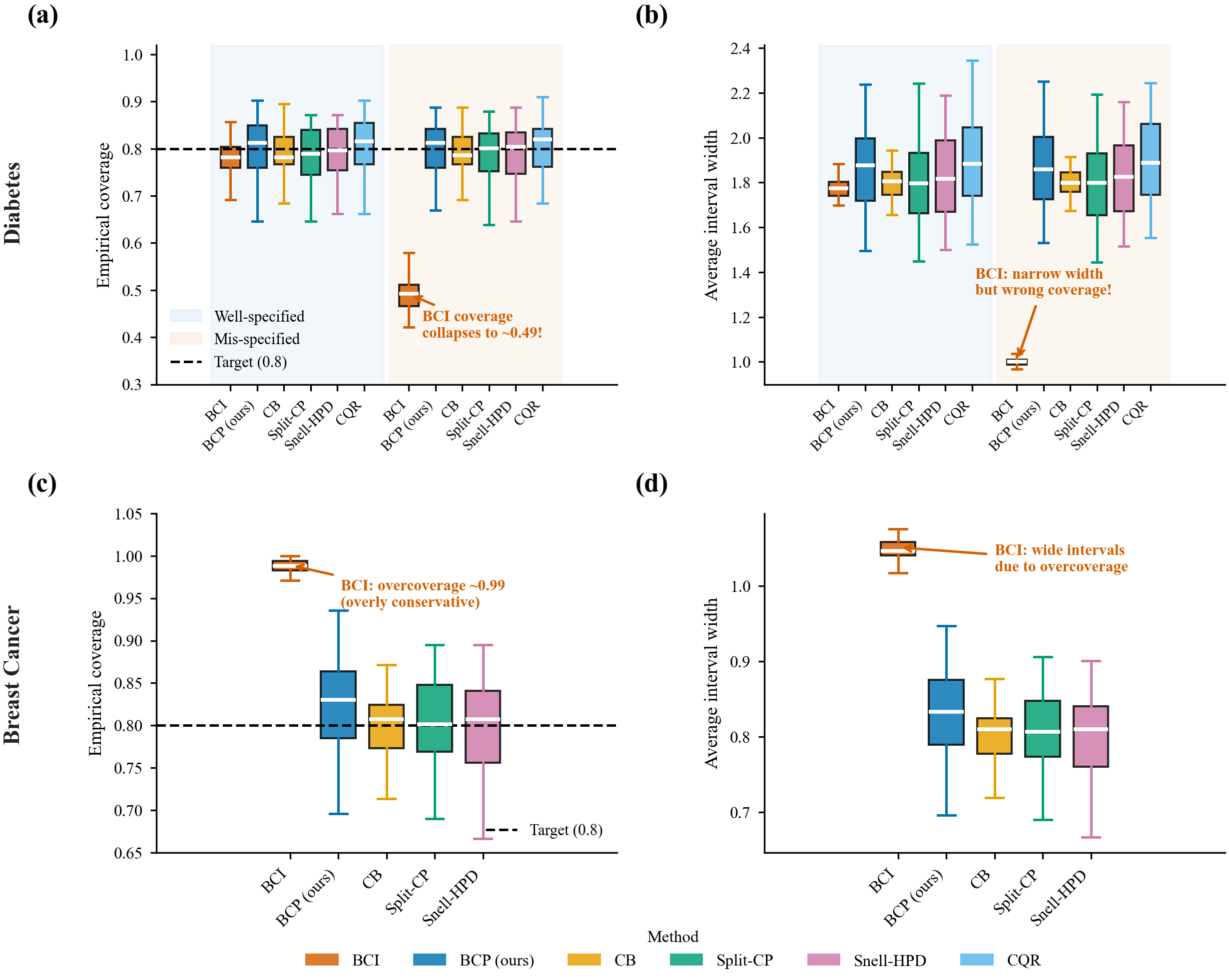}
\caption{Prediction interval comparisons on regression and classification.
(a)~Diabetes coverage: BCI severely under-covers under misspecification.
(b)~Diabetes interval width: BCI is narrow but invalid.
(c)~Breast Cancer coverage: BCI substantially overcovers.
(d)~Breast Cancer width: BCI's overcoverage yields excessively wide sets.
Dashed lines mark the $80\%$ target.}
\label{fig:reg_comparison}
\end{figure}

\subsection{Breast Cancer Classification}
\label{sec:classification}
We evaluate sparse binary classification on the Wisconsin Breast Cancer
dataset \citep{wolberg1990multisurface}: $569$ samples, $30$ features, binary
labels. The model is Bayesian logistic regression with standard Gaussian
priors. We fix $\alpha{=}0.2$, $\beta{=}0.2$ across 50 random splits
(52.5 / 17.5 / 30 partition) with $T{=}8000$ MCMC samples.
Average set size $|C(x)|\in\{0,1,2\}$; values below $1$ indicate some inputs
receive empty sets.

\begin{table}[h]
\centering
\caption{Classification results (50 random splits; $\alpha{=}0.2$,
$\beta{=}0.2$, confidence $1{-}\beta{=}80\%$).
\textcolor{red}{Red} marks substantial overcoverage.}
\label{tab:class_compare}
\begin{tabular}{lccccc}
\toprule
\textbf{Method} & \textbf{Guarantee}
  & \textbf{Cov.\,(\%)} & \textbf{Avg.\ set size}
  & \textbf{Pred.\,(s)} & \textbf{Calib.\,(s)} \\
\midrule
Split-CP  & marginal & $80.5\pm5.1$ & $0.809\pm0.052$ & $<10^{-4}$ & $0.005$ \\
BCI       & none     & \textcolor{red}{$98.9\pm0.8$} & $1.051\pm0.017$
          & $<10^{-4}$ & --- \\
CB        & marginal & $80.3\pm3.8$ & $0.806\pm0.037$ & $1.0\!\times\!10^{-3}$ & $9.432$ \\
Snell-HPD & PAC      & $79.9\pm5.6$ & $0.804\pm0.057$ & $<10^{-4}$ & $0.130$ \\
\textbf{BCP} & \textbf{PAC}
          & $\boldsymbol{82.5\pm5.4}$ & $\boldsymbol{0.829\pm0.056}$
          & $\boldsymbol{<10^{-4}}$   & $\boldsymbol{0.226}$ \\
\bottomrule
\end{tabular}
\end{table}

BCI overcovers at $98.9\%$ when the target coverage is set to $80\%$, confirming that Bayesian credible intervals without conformal calibration are unreliable.
Among PAC-based methods, BCP ($82.5\%$, set size $0.829$) again ensures target coverage compared to Snell-HPD ($79.9\%$, $0.804$) at calibration
overhead ($0.226$\,s vs.\ $0.130$\,s), attributable to the BQ step.
Split-CP and CB perform comparably; CB offers no efficiency advantage despite
its higher cost as a full CP method. BCP responds to both $\alpha$ and $\beta$, with tighter constraints yielding more conservative but valid prediction sets, as
detailed in Appendix~\ref{app:sensitivity}.

\subsection{Synthetic Multimodal Regression}
\label{sec:multimodal}
We test BCP in a setting where posterior predictive geometry is essential.
We construct a 1D bimodal regression task,
\begin{equation}
    Y \mid x,\theta
    \;\sim\;
    \tfrac{1}{2}\,\mathcal{N}(\theta_1 x,\,\sigma^2)
    \;+\;
    \tfrac{1}{2}\,\mathcal{N}(\theta_1 x+\delta,\,\sigma^2),
\end{equation}
with $\theta_1{=}1$, $\delta{=}4$, $\sigma{=}0.4$, using 200 training,
100 calibration, and 200 test points, with a Bayesian Gaussian mixture fit
via MCMC\@.
The posterior predictive is bimodal at every test input; the ideal prediction set is disjoint.

\paragraph{Role of BQ in multimodal settings.}
Although $g(\lambda)=\mathbb{E}_x[|C_\mathrm{HPD}(x;\lambda)|]$ is monotone
non-increasing in $\lambda$, the geometry of HPD sets changes abruptly when
$\lambda$ crosses the density level of the secondary mode peak: a connected
component disappears, inducing a sharp increase in estimator variance under
finite posterior samples.
A naive grid search can therefore straddle this discontinuity and return an
unstable $\lambda^*$; BQ places a GP prior over $\lambda\mapsto g(\lambda)$,
reducing estimation variance precisely where it is highest and stabilising
threshold selection near mode boundaries (see Appendix~\ref{app:hpd-bq}).
Figure~\ref{fig:multimodal_discontinuity} illustrates this effect.

\begin{figure}[t]
\centering
\includegraphics[width=0.8\linewidth]{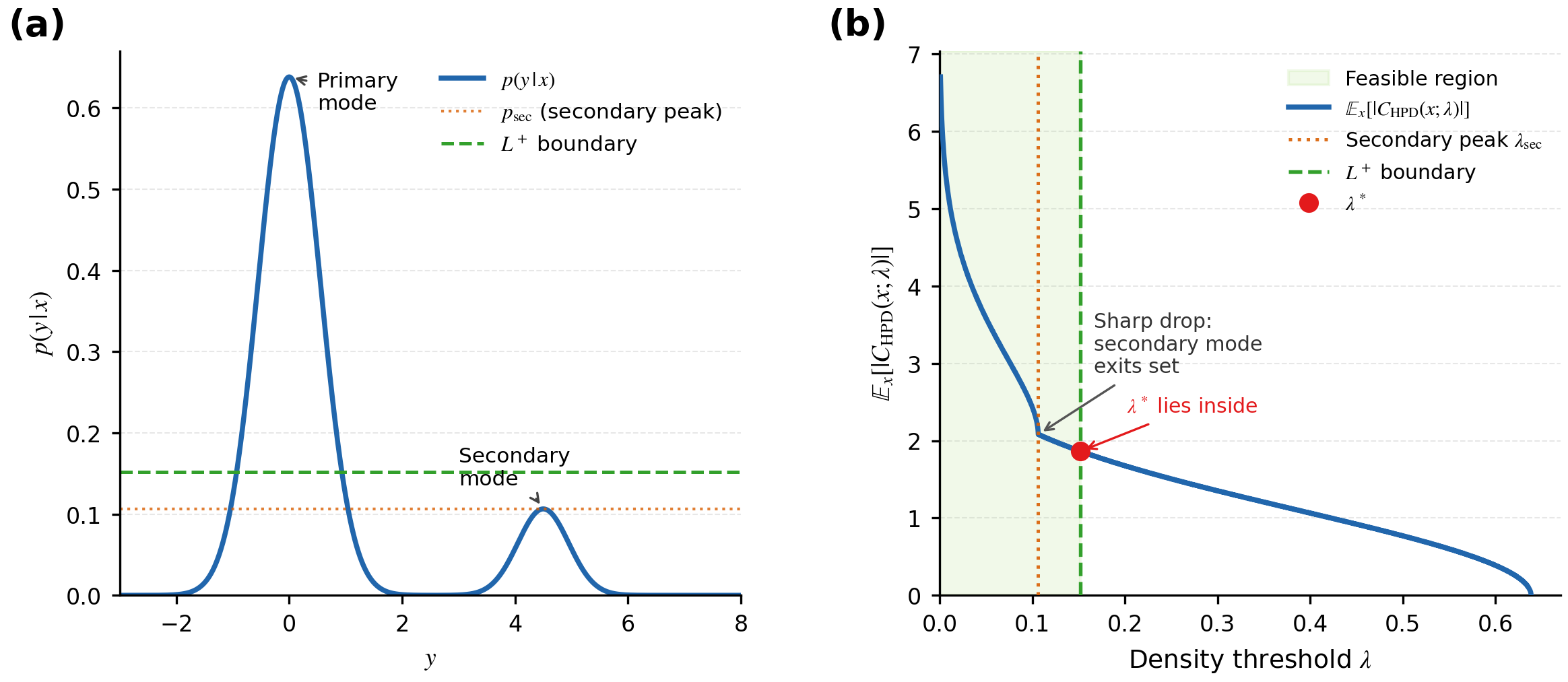}
\caption{Expected HPD set size as a function of $\lambda$.
The sharp drop occurs as $\lambda$ crosses the secondary-mode peak density,
causing an abrupt change in geometry in the HPD set.
The green dashed line marks the $L^+$ feasibility boundary; the red dot
marks $\lambda^*$.
This instability in the objective estimate motivates GP-BQ for stable
threshold selection.}
\label{fig:multimodal_discontinuity}
\end{figure}

\begin{table}[h]
\centering
\caption{Synthetic multimodal regression ($\alpha{=}0.2$, $\beta{=}0.2$;
200 test points).
``Disjoint'' indicates whether non-contiguous sets are possible.
``PAC'' denotes the proportion of trials in which empirical coverage exceeds
$1{-}\alpha$, with target level ${\ge}1{-}\beta$.}
\label{tab:multimodal}
\resizebox{\linewidth}{!}{%
\begin{tabular}{lcccccc}
\toprule
\textbf{Method} & \textbf{Guarantee} & \textbf{Disjoint}
  & \textbf{Size} & \textbf{Cov.\,(\%)} & \textbf{PAC} & \textbf{Time\,(s)} \\
\midrule
Split-CP  & marginal & no  & $4.82\pm0.14$ & $76.0\pm5.2$ & $0.32\pm0.07$ & $0.002$ \\
CQR       & marginal & no  & $4.67\pm0.11$ & $80.0\pm4.4$ & $0.44\pm0.07$ & $0.207$ \\
Snell-HPD & PAC      & yes & $2.06\pm0.14$ & $82.1\pm4.3$ & $0.78\pm0.06$ & $0.212$ \\
\textbf{BCP} & \textbf{PAC} & \textbf{yes}
  & $\mathbf{2.07\pm0.14}$ & $\mathbf{82.6\pm4.2}$
  & $\mathbf{0.80\pm0.06}$ & $\mathbf{0.221}$ \\
\bottomrule
\end{tabular}}
\end{table}

The dominant efficiency gain comes from HPD set geometry.
Interval-constrained methods (Split-CP and CQR) produce connected
prediction intervals that must span the low-density valley between modes,
yielding mean set sizes of $4.82$ and $4.67$ respectively.
HPD-based methods avoid this penalty by recovering the bimodal structure
directly, reducing the mean set size to approximately $2.07$; this reduction
is statistically significant relative to both interval baselines
($p < 0.0001$, Wilcoxon signed-rank test).

BCP and Snell-HPD achieve comparable set sizes ($2.07$ vs~$2.06$), not
significantly different. Snell-HPD was not originally designed for multimodal settings; we include Snell-HPD as a direct baseline since both methods use HPD set construction. The strong performance in this multimodal setting highlights that the efficiency gain is primarily driven by the shared geometric properties of HPD sets.
The distinction between the two methods lies in PAC calibration:
BCP achieves a PAC pass rate of $0.80$, exactly meeting the target
$1{-}\beta=0.80$, while Snell-HPD reaches $0.78$, falling slightly below
the target level.
Split-CP and CQR reach only $0.32$ and $0.44$, as marginal quantile
calibration does not account for the heavier tails of the bimodal predictive
distribution.

In terms of runtime, BCP ($0.221$\,s) is comparable to Snell-HPD
($0.212$\,s) and scales linearly with calibration set size
(Appendix~\ref{app:scalability}).

\subsection{High-Dimensional Classification on ImageNet-A}
\label{sec:imagenet}
We evaluate BCP on ImageNet-A \citep{hendrycks2021nae} as a high-dimensional sanity check and scalability experiment, rather than as a setting where BCP is expected to outperform Snell-HPD. ImageNet-A is a natural distribution-shift benchmark with $200$ different classes. A frozen ImageNet-pretrained ResNet-50 serves as the feature extractor, with a Monte Carlo Dropout head ($T{=}30$, dropout rate $0.3$). We use 2{,}000 training, 2{,}000 calibration, and 3{,}000 test samples over five random splits, with $(\alpha,\beta){=}(0.2,0.2)$.

\begin{table}[h]
\centering
\caption{ImageNet-A results under distribution shift
($\alpha{=}0.2$, $\beta{=}0.2$, $T{=}30$, 5 splits). All PAC-based methods use AOI scores and $L^+$ enforcement. Calib.\ Time reports calibration wall-clock time per split (AOI score pre-computation
excluded).}
\label{tab:imagenet_stats}
\begin{tabular}{lccccc}
\toprule
\textbf{Method} & \textbf{Guarantee}
  & \textbf{Coverage\,(\%)}
  & \textbf{Mean $|C(x)|$} & \textbf{$\mathrm{P}_{95}$ $|C(x)|$}
  & \textbf{Calib.\,Time\,(s)} \\
\midrule
Split-CP  & marginal & $77.78\pm0.42$ & $122.7\pm5.1$ & $138$ & ${<}0.01$ \\
CB        & marginal & $79.34\pm0.41$ & $123.1\pm5.0$ & $138$ & ${<}0.01$ \\
Snell-HPD & PAC      & $80.00\pm0.12$ & $40.3\pm2.8$  & $52$  & $3.23$ \\
BCP       & PAC      & $80.00\pm0.12$ & $40.3\pm2.8$  & $52$  & $9.65$ \\
\bottomrule
\end{tabular}
\end{table}

Under distribution shift, all methods maintain valid coverage, while PAC-based methods achieve substantially smaller prediction sets. Both BCP and Snell-HPD yield mean $|C|{=}40.3$ ($\mathrm{P}_{95}{=}52$) versus above $122$ ($\mathrm{P}_{95}{=}138$) for Split-CP and CB; this efficiency gain is attributable to the AOI score function rather than to the PAC constraint itself, as confirmed by an MSP-score ablation (Appendix~\ref{app:msp-ablation}). These results confirm that the BCP pipeline remains feasible in a high-dimensional classification setting.

BCP and Snell-HPD produce identical results in this setting, as expected. In multiclass classification, the expected set size
$g(\lambda) = \mathbb{E}[|C(X;\lambda)|]$ is monotonically increasing in $\lambda$, since a larger threshold admits more classes into the prediction set. As a result, the optimal solution is the smallest feasible $\lambda$ satisfying the $L^+$ constraint, which coincides with the threshold selected by Snell-HPD. ImageNet-A should therefore be read as a scalability and consistency check, not as a setting where BCP is expected to yield additional efficiency gains.

In terms of computation, the dominant cost is AOI score construction, which requires $T{=}30$ stochastic forward passes through the MC Dropout head and takes ${\sim}9$\,s per split. Once the $N_{\mathrm{cal}}{\times}K$ score matrix is cached, evaluating $g(\lambda)$ at any threshold reduces to a single $O(N_{\mathrm{cal}} \cdot K)$ matrix comparison, making threshold selection negligible for all methods.

\begin{figure}[h]
\centering
\includegraphics[width=\linewidth]{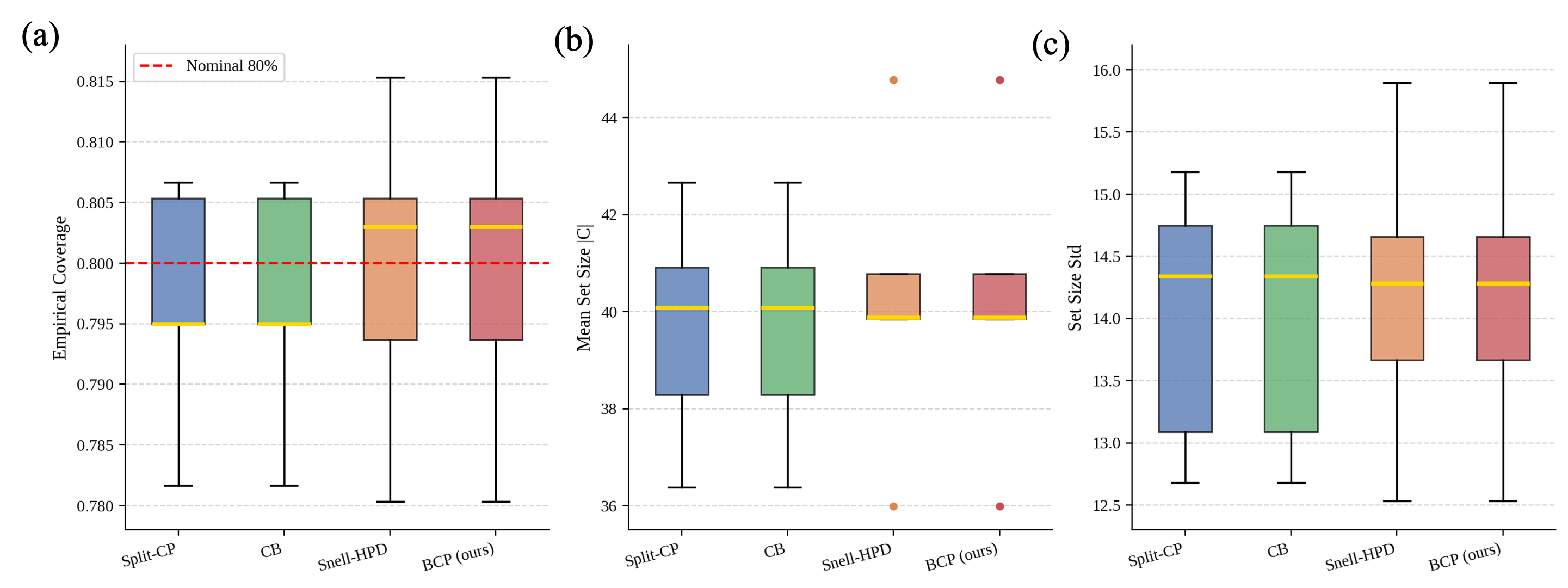}
\caption{ImageNet-A results over five random splits.
(a)~Empirical coverage; red dashed at $80\%$.
(b)~Mean set size $|C|$.
(c)~Set size standard deviation across test examples.}
\label{fig:imagenet-results}
\end{figure}

\paragraph{Summary.}
BCP consistently achieves reliable coverage while adapting to the structure of the prediction set, improving efficiency across diverse settings. Under model misspecification, conformal calibration via $L^+$ restores nominal coverage where Bayesian credible intervals fail, while maintaining robustness through exchangeability. In the classification setting, where Bayesian credible intervals severely overcover ($98.9\%$), BCP provides explicit efficiency-validity control via $(\alpha,\beta)$, achieving near-target coverage with only minor computational overhead (${<}0.3$\,s calibration across all small-scale experiments).

The largest gains arise in multimodal settings, where BCP exploits disjoint
predictive structure to substantially reduce the prediction set size while exactly
meeting PAC guarantees. In contrast, under unimodal classification tasks such
as ImageNet-A, where the expected set size is monotone in the threshold,
BCP recovers the same solution as existing PAC-based methods, indicating that
optimisation yields benefits primarily when prediction sets exhibit complex
structure. Overall, BCP combines statistical validity with adaptive efficiency, delivering robust performance across misspecification, distribution shift, and structural
complexity.

\section{Conclusion}
\label{sec:conclusion}
We introduced Bayesian Conformal Prediction (BCP), a PAC-controlled framework
that uses Bayesian posterior predictive structure to construct prediction sets
with adaptive geometry. Validity rests on exchangeability and the $L^+$ risk constraint, while posterior predictive scores shape the set geometry and Bayesian quadrature stabilises objective estimation near mode-boundary density
levels.

Our results show that BCP is most useful when predictive uncertainty is
multimodal. In threshold-based settings with nested prediction sets, such as unimodal regression or standard multiclass classification, BCP recovers the smallest feasible threshold and matches existing PAC-based methods. In multimodal settings,
HPD geometry yields disjoint prediction sets that avoid low-density
inter-modal regions, leading to substantial efficiency gains.

\textbf{Limitations and broader impact.}
BCP inherits the computational cost of posterior inference, requiring $T$
posterior samples and AOI reweighting per calibration point, and scaling to very large datasets or high-dimensional label spaces remains open. Its guarantees also rely on
exchangeability and may degrade under severe distribution shift. As a
general-purpose uncertainty quantification tool, BCP can reduce overconfident
predictions in high-stakes applications such as medical diagnosis and
autonomous systems, but its assumptions must be checked in practice.
\section*{Acknowledgments}
This work was supported by the Engineering and Physical Sciences Research Council (EPSRC) under grant[EP/Y030826/1]. 
The author thanks the supervisory team, Prof Samuel Kaski and Dr Michele Caprio.
\bibliography{reference}

@book{vovk2005algorithmic,
  title={Algorithmic Learning in a Random World},
  author={Vovk, Vladimir and Gammerman, Alex and Shafer, Glenn},
  year={2005},
  publisher={Springer}
}

@article{shafer2008tutorial,
  title={A Tutorial on Conformal Prediction},
  author={Shafer, Glenn and Vovk, Vladimir},
  journal={Journal of Machine Learning Research},
  volume={9},
  pages={371--421},
  year={2008},
  month={Mar}
}

@article{angelopoulos2021gentle,
  author       = {Anastasios N. Angelopoulos and
                  Stephen Bates},
  title        = {A Gentle Introduction to Conformal Prediction and Distribution-Free
                  Uncertainty Quantification},
  journal      = {CoRR},
  volume       = {abs/2107.07511},
  year         = {2021},
  url          = {https://arxiv.org/abs/2107.07511},
  eprinttype   = {arXiv},
  eprint       = {2107.07511},
  bibsource    = {dblp computer science bibliography, https://dblp.org}
}

@inproceedings{
fong2021conformalbayesiancomputation,
title={Conformal Bayesian Computation},
author={Edwin Fong and Christopher C. Holmes},
booktitle={Advances in Neural Information Processing Systems},
editor={A. Beygelzimer and Y. Dauphin and P. Liang and J. Wortman Vaughan},
year={2021},
url={https://openreview.net/forum?id=e95xWqO7ehi}
}

@article{efron2004least,
  title={Least angle regression},
  author={Efron, Bradley and Hastie, Trevor and Johnstone, Iain and Tibshirani, Robert},
  journal={Annals of Statistics},
  volume={32},
  number={2},
  pages={407--499},
  year={2004},
  publisher={Institute of Mathematical Statistics}
}

@mastersthesis{jansen2013robust,
  title={Robust Bayesian inference under model misspecification},
  author={Jansen, Laurens},
  school={Leiden University},
  year={2013},
  note={Chapter 4.5}
}

@article{wolberg1990multisurface,
  title={Multisurface method of pattern separation for medical diagnosis applied to breast cytology},
  author={Wolberg, William H. and Mangasarian, Olvi L.},
  journal={Proceedings of the National Academy of Sciences},
  volume={87},
  number={23},
  pages={9193--9196},
  year={1990},
  publisher={National Academy of Sciences}
}

@book{bernardo2009bayesian,
  title={Bayesian theory},
  author={Bernardo, Jose M and Smith, Adrian FM},
  volume={405},
  year={2009},
  publisher={John Wiley \& Sons}
}

@article{
caprio2025conformalized,
title={Conformalized Credal Regions for Classification with Ambiguous Ground Truth},
author={Michele Caprio and David Stutz and Shuo Li and Arnaud Doucet},
journal={Transactions on Machine Learning Research},
issn={2835-8856},
year={2025},
url={https://openreview.net/forum?id=L7sQ8CW2FY},
note={}
}

@article{stutz2022learning,
  author       = {David Stutz and
                  Krishnamurthy Dvijotham and
                  Ali Taylan Cemgil and
                  Arnaud Doucet},
  title        = {Learning Optimal Conformal Classifiers},
  journal      = {CoRR},
  volume       = {abs/2110.09192},
  year         = {2021},
  url          = {https://arxiv.org/abs/2110.09192},
  eprinttype   = {arXiv},
  eprint       = {2110.09192},
  bibsource    = {dblp computer science bibliography, https://dblp.org}
}

@article{
caprio2024credal,
title={Credal Bayesian Deep Learning},
author={Michele Caprio and Souradeep Dutta and Kuk Jin Jang and Vivian Lin and Radoslav Ivanov and Oleg Sokolsky and Insup Lee},
journal={Transactions on Machine Learning Research},
issn={2835-8856},
year={2024},
url={https://openreview.net/forum?id=4NHF9AC5ui},
note={}
}

@inproceedings{gal2016dropout,
  title={Dropout as a bayesian approximation: Representing model uncertainty in deep learning},
  author={Gal, Yarin and Ghahramani, Zoubin},
  booktitle={international conference on machine learning},
  pages={1050--1059},
  year={2016},
  organization={PMLR}
}

@article{tyralis2024review,
  title={A review of predictive uncertainty estimation with machine learning},
  author={Tyralis, Hristos and Papacharalampous, Georgia},
  journal={Artificial Intelligence Review},
  volume={57},
  number={4},
  pages={94},
  year={2024},
  publisher={Springer}
}

@inproceedings{
snell2025conformal,
title={Conformal Prediction as Bayesian Quadrature},
author={Jake C. Snell and Thomas L. Griffiths},
booktitle={Forty-second International Conference on Machine Learning},
year={2025},
url={https://openreview.net/forum?id=PNmkjIzHB7}
}

@article{barber2023conformal,
  title={Conformal prediction beyond exchangeability},
  author={Barber, Rina Foygel and Candes, Emmanuel J and Ramdas, Aaditya and Tibshirani, Ryan J},
  journal={The Annals of Statistics},
  volume={51},
  number={2},
  pages={816--845},
  year={2023},
  publisher={Institute of Mathematical Statistics}
}

@misc{bariletto2025conformalized,
      title={Conformalized Bayesian Inference, with Applications to Random Partition Models}, 
      author={Nicola Bariletto and Nhat Ho and Alessandro Rinaldo},
      year={2025},
      eprint={2511.05746},
      archivePrefix={arXiv},
      primaryClass={stat.ME},
      url={https://arxiv.org/abs/2511.05746}, 
}

@article{ohagan1991bayes,
  title   = {Bayes--Hermite Quadrature},
  author  = {O'Hagan, Anthony},
  journal = {Journal of Statistical Planning and Inference},
  volume  = {29},
  number  = {3},
  pages   = {245--260},
  year    = {1991}
}

@article{lei2018distributionfree,
author = {Jing Lei and Max G’Sell and Alessandro Rinaldo and Ryan J. Tibshirani and Larry Wasserman},
title = {Distribution-Free Predictive Inference for Regression},
journal = {Journal of the American Statistical Association},
volume = {113},
number = {523},
pages = {1094--1111},
year = {2018},
publisher = {Taylor \& Francis},
doi = {10.1080/01621459.2017.1307116},

URL = { 
    
        https://doi.org/10.1080/01621459.2017.1307116
    
    
},
eprint = { 
    
        https://doi.org/10.1080/01621459.2017.1307116
    

}

}

@misc{angelopoulos2025conformalriskcontrol,
      title={Conformal Risk Control}, 
      author={Anastasios N. Angelopoulos and Stephen Bates and Adam Fisch and Lihua Lei and Tal Schuster},
      year={2025},
      eprint={2208.02814},
      archivePrefix={arXiv},
      primaryClass={stat.ME},
      url={https://arxiv.org/abs/2208.02814}, 
}

@misc{bellotti2021optimized,
      title={Optimized conformal classification using gradient descent approximation}, 
      author={Anthony Bellotti},
      year={2021},
      eprint={2105.11255},
      archivePrefix={arXiv},
      primaryClass={cs.LG},
      url={https://arxiv.org/abs/2105.11255}, 
}

@article{sadinle2018least,
   title={Least Ambiguous Set-Valued Classifiers With Bounded Error Levels},
   volume={114},
   ISSN={1537-274X},
   url={http://dx.doi.org/10.1080/01621459.2017.1395341},
   DOI={10.1080/01621459.2017.1395341},
   number={525},
   journal={Journal of the American Statistical Association},
   publisher={Informa UK Limited},
   author={Sadinle, Mauricio and Lei, Jing and Wasserman, Larry},
   year={2018},
   month=June, pages={223–234} }

@inbook{romano2019conformalized,
author = {Romano, Yaniv and Patterson, Evan and Cand\`{e}s, Emmanuel J.},
title = {Conformalized quantile regression},
year = {2019},
publisher = {Curran Associates Inc.},
address = {Red Hook, NY, USA},
booktitle = {Proceedings of the 33rd International Conference on Neural Information Processing Systems},
articleno = {318},
numpages = {11}
}

@article{hendrycks2021nae,
  title={Natural Adversarial Examples},
  author={Dan Hendrycks and Kevin Zhao and Steven Basart and Jacob Steinhardt and Dawn Song},
  journal={CVPR},
  year={2021},
  url={https://arxiv.org/abs/1907.07174}, 
}

@book{gelman2013bayesian,
  title     = {Bayesian Data Analysis},
  author    = {Gelman, Andrew and Carlin, John B. and Stern, Hal S. and Dunson, David B. and Vehtari, Aki and Rubin, Donald B.},
  edition   = {3},
  year      = {2013},
  publisher = {Chapman and Hall/CRC},
  doi       = {10.1201/b16018}
}

\newpage
\appendix
\appendix
\section{Theoretical Background and Derivations}
\label{app:theory}

\subsection{Efficiency of Conformal Predictors.}
Throughout this appendix, prediction regions are parameterised by a scalar threshold $\lambda \in \mathbb{R}$, and we write $\mathcal{C}(x;\lambda)$ for the induced conformal prediction set. Let $\mathcal{C}(x;\lambda)$ denote a prediction region (or confidence predictor) that maps
an input $x$ to a subset of possible outputs.

A classical result by \citet{vovk2005algorithmic} states that for any valid predictor $\mathcal{C}(x;\lambda)$ defined on a standard Borel space $\mathcal{Z}$, there exists a conformal predictor $\mathcal{C}'(x;\lambda)$ such that
\begin{equation}
|\mathcal{C}'(x;\lambda)| \leq |\mathcal{C}(x;\lambda)|,
\quad \forall x \in \mathcal{X},
\end{equation}
while preserving the same marginal coverage guarantee,
\begin{equation}
\mathbb{P}\!\left[Y \in \mathcal{C}'(X;\lambda)\right] \ge 1 - \alpha,
\end{equation}
where $|\mathcal{C}(x;\lambda)|$ denotes the size (e.g., interval length or set cardinality) of the prediction region and $\alpha$ is the user-specified miscoverage level.
This result implies that, among all predictors achieving a given coverage level, there exists a conformal predictor whose prediction regions are almost surely no larger than those of any other predictor.

\subsection{Expectation-based Formulation.}
Both coverage and efficiency can be expressed in terms of expectation.
The marginal coverage can be written as
\begin{equation}
\mathbb{P}\!\left[Y \in \mathcal{C}(X;\lambda)\right]
=
\mathbb{E}_{X}
\!\left[
\mathbb{E}_{Y}
\!\left[
\mathbb{I}\{Y \in \mathcal{C}(X;\lambda)\} \mid X
\right]
\right],
\end{equation}
while efficiency is measured by the expected prediction set size
\begin{equation}
\mathrm{Eff}(\lambda) := \mathbb{E}\!\left[|\mathcal{C}(X;\lambda)|\right].
\end{equation}
This leads to the following constrained optimisation problem:
\begin{equation}
\min_{\lambda} \; \mathbb{E}\!\left[|\mathcal{C}(X;\lambda)|\right]
\quad \text{s.t.} \quad
\mathbb{P}\!\left[Y \in \mathcal{C}(X;\lambda)\right] \ge 1 - \alpha.
\end{equation}

This formulation provides the decision-theoretic foundation for our method.
In the main text, this marginal coverage constraint is further strengthened to a PAC-style guarantee via conformal risk control, and the expectation over $X$ is estimated using Bayesian quadrature.

\section{Theoretical Foundations of the Risk Formulation}
\label{app:risk-theory}

\subsection{Split Conformal Prediction as Risk Control}

Let $z = (x,y)$ denote an input--output pair with $x \in \mathcal{X}$ and $y \in \mathcal{Y}$, and let $\mathcal{D} = \{(X_i,Y_i)\}_{i=1}^n$ denote the observed data.
We consider prediction regions parameterised by a scalar threshold $\lambda$, denoted by $\mathcal{C}(x;\lambda)$.

We define the miscoverage loss
\begin{equation}
L(y,\mathcal{C}(x;\lambda)) = \mathbb{I}\{ y \notin C(x; \lambda) \},
\end{equation}
and the corresponding \emph{data-dependent miscoverage risk}
\begin{equation}
R_{\mathcal{D}}(\lambda)
=
\mathbb{P}_{(X,Y)}\!\left(Y \notin \mathcal{C}(X;\lambda)\right),
\label{eq:data-risk}
\end{equation}
where the probability is taken with respect to a fresh test pair $(X,Y)$ drawn exchangeably with the calibration data, conditional on the training data.

In split conformal prediction, the non-conformity score function is fixed with respect to the calibration labels, and the threshold $\lambda$ is selected using the calibration data only.
Rather than enforcing marginal coverage deterministically, we adopt the conformal risk control (CRC) framework \citep{angelopoulos2025conformalriskcontrol, snell2025conformal}, which guarantees that
\begin{equation}
\mathbb{P}_{\mathcal D}\!\left(
R_{\mathcal D}(\lambda)\le \alpha
\right)\ge 1-\beta,
\label{eq:crc-app}
\end{equation}
for user-specified risk level $\alpha$ and confidence parameter $\beta$.
This PAC-style guarantee forms the basis of the coverage control used throughout the main text.

\subsection{Bayesian Non-conformity Scores under Split Conformal Prediction}
\label{app: Valid Bayesian AOI}
Classical non-conformity scores often ignore uncertainty in model parameters.
In this work, we define a Bayesian predictive score based on a posterior trained exclusively on the training data:
\begin{equation}
s(x,y)
=
-\log p(y \mid x, \mathcal{D}_{\mathrm{tr}})
=
-\log \int f_\theta(y \mid x)\, \pi(\theta \mid \mathcal{D}_{\mathrm{tr}})\,
\mathrm{d}\theta,
\label{eq:bayes-score-split}
\end{equation}
where $\mathcal{D}_{\mathrm{tr}}$ denotes the training dataset and
$\pi(\theta \mid \mathcal{D}_{\mathrm{tr}})$ is the corresponding posterior.
Importantly, this score function is fixed with respect to the calibration and test labels, satisfying the fundamental requirement of split conformal
prediction.

\begin{proposition}[Validity of Bayesian Scores under Split CP]
\label{prop:bayes-split-validity}
Assume that the calibration and test samples are exchangeable conditional on the training data, and that the score function~\eqref{eq:bayes-score-split}
is fixed with respect to the calibration labels.
Let $\lambda$ be selected using conformal risk control applied to the resulting calibration losses.
Then the resulting conformal predictor satisfies the PAC-style guarantee
\eqref{eq:crc-app}.
\end{proposition}

\paragraph{Proof.}
Let the training data be $\mathcal{D}_{\mathrm{tr}}$, the calibration data be
$\mathcal{D}_{\mathrm{cal}}=\{Z_i\}_{i=1}^n$ with $Z_i=(X_i,Y_i)$, and the test point be
$Z_{n+1}=(X_{n+1},Y_{n+1})$.
By assumption, conditional on $\mathcal{D}_{\mathrm{tr}}$, the random variables
$Z_1,\dots,Z_n,Z_{n+1}$ are exchangeable.

Define the Bayesian score function
$s(x,y) = -\log p(y\mid x,\mathcal{D}_{\mathrm{tr}})$ as in~\eqref{eq:bayes-score-split}.
Since $\pi(\theta\mid \mathcal{D}_{\mathrm{tr}})$ depends only on $\mathcal{D}_{\mathrm{tr}}$,
the score function $s$ is $\sigma(\mathcal{D}_{\mathrm{tr}})$-measurable and does not depend
on any calibration labels $\{Y_i\}_{i=1}^n$.
Hence, for any threshold $\lambda$, the (miscoverage) loss on an example $z=(x,y)$,
\[
\ell(z,\lambda) := \mathbb{I}\{ y \notin C(x;\lambda)\},
\quad C(x;\lambda)=\{y'\in\mathcal{Y}: s(x,y')\le \lambda\},
\]
is also $\sigma(\mathcal{D}_{\mathrm{tr}})$-measurable as a function of $z$ and $\lambda$.

Therefore, conditional on $\mathcal{D}_{\mathrm{tr}}$, the random variables
$\ell(Z_1,\lambda),\dots,\ell(Z_n,\lambda),\ell(Z_{n+1},\lambda)$ are exchangeable for each
fixed $\lambda$.
In particular, the calibration losses $\{\ell(Z_i,\lambda)\}_{i=1}^n$ are i.i.d. conditional
on $\mathcal{D}_{\mathrm{tr}}$ (equivalently, exchangeable with a common conditional law).

Let $\widehat{\lambda}=\widehat{\lambda}(\mathcal{D}_{\mathrm{cal}})$ be the threshold selected by conformal risk control (CRC) applied to the calibration losses induced by $s$.
By the CRC guarantee of \citet{angelopoulos2025conformalriskcontrol} (applied conditional on $\mathcal{D}_{\mathrm{tr}}$),
we have
\[
\mathbb{P}_{\mathcal{D}_{\mathrm{cal}}}\!\left(
R(\widehat{\lambda}) \le \alpha \;\middle|\; \mathcal{D}_{\mathrm{tr}}
\right) \ge 1-\beta,
\]
where $R(\lambda)=\mathbb{P}(Y\notin C(X;\lambda)\mid \mathcal{D}_{\mathrm{tr}})$ denotes the (test-time) miscoverage risk under the exchangeable draw of $(X,Y)$.
Removing the conditioning on $\mathcal{D}_{\mathrm{tr}}$ yields the PAC-style statement
\eqref{eq:crc-app}.
\hfill$\square$

\subsection{Bayesian Risk as an Optimisation Objective}

While coverage is enforced via conformal risk control, the Bayesian structure is used to guide the selection of the threshold $\lambda$ through an efficiency-oriented objective.
We define the Bayesian decision risk
\begin{equation}
\mathcal{R}(\lambda)
=
\mathbb{E}_{\theta \sim \pi(\cdot \mid \mathcal{D}_{\mathrm{tr}})}
\mathbb{E}_{X,Y \sim p(\cdot \mid \theta)}
\!\left[
L(Y,\mathcal{C}(X;\lambda))
\right],
\label{eq:bayes-risk}
\end{equation}
which averages miscoverage over posterior uncertainty and the test input
distribution.

Minimising the expected prediction set size
\begin{equation}
\mathbb{E}_X\!\left[|\mathcal{C}(X;\lambda)|\right]
\end{equation}
under the constraint~\eqref{eq:crc-app} yields the decision-theoretic
optimisation problem studied in the main text.
In practice, the expectation of $X$ is approximated using Bayesian quadrature, while the Bayesian risk~\eqref{eq:bayes-risk} serves solely as an optimisation criterion, and does not provide coverage guarantees.

\section{Bayesian Risk Estimation and AOI Implementation}
\label{app:method-theory}





\subsection{Exchangeability and Validity}
\label{Appendix: validity}

Because the posterior $\pi(\theta\mid\mathcal{D}_{\mathrm{tr}})$ is fixed, the induced score function $s(x,y)$ in~\eqref{eq:score} is applied identically to all calibration and test points.
Assuming exchangeability between the calibration and test samples conditional on the training data, the calibration losses are i.i.d. given $\mathcal{D}_{\mathrm{tr}}$.

Consequently, the conformal risk control guarantees derived in Section~\ref{app:risk-theory} apply directly, and the resulting prediction sets satisfy the PAC-style coverage guarantee
\begin{equation}
\mathbb{P}_{\mathcal{D}}\!\left(
\mathbb{P}(Y\notin\mathcal{C}(X;\lambda))\le\alpha
\right)\ge 1-\beta.
\end{equation}

\subsection{Add-one-in (AOI)\label{Appendix:AOI} Importance Sampling Approximation} AOI approximates posterior predictive densities via reweighting of posterior draws:
\begin{equation}
    \pi(\theta \mid \mathcal{D}_{\mathrm{tr}} \cup \{(x_{n+1},y)\})
\propto 
\pi(\theta \mid \mathcal{D}_{\mathrm{tr}})\, f_\theta(y \mid x_{n+1}),
\quad
\tilde{w}^{(t)} = \frac{f_{\theta^{(t)}}(y \mid x_{n+1})}
{\sum_{t'} f_{\theta^{(t')}}(y \mid x_{n+1})}.
\end{equation}
The posterior predictive under the augmented dataset becomes
\begin{equation}
    \widehat{p}(Y_i \mid X_i, \mathcal{D}_{\mathrm{tr}})
    = \sum_{t=1}^T \tilde{w}^{(t)} f_{\theta^{(t)}}(Y_i \mid X_i).
\end{equation}
Each score $s_i = -\log \widehat{p}(Y_i \mid X_i, \mathcal{D}_{\mathrm{tr}})$
is thus computed symmetrically for all samples.




\subsection{Risk Control via $L^+$}
Conformal Risk Control estimates the empirical risk
$\hat{R}_n(\lambda) = n^{-1} \sum_i \ell(z_i,\lambda)$,
augmented by the stochastic bound
\begin{equation}
    L^+ = \sum_{i=1}^{n+1} U_i \ell_{(i)}, 
\quad U \sim \mathrm{Dir}(1,\dots,1),
\end{equation}
which stochastically dominates the posterior risk.  
The valid threshold satisfies
\begin{equation}
\mathbb{P}_{\mathcal{D}}\!\left(
\mathbb{P}_{(X,Y)}\!\left(
Y\notin\mathcal{C}(X;\lambda)
\right)\le\alpha
\right)\ge 1-\beta
,
\end{equation}
guaranteeing $(1-\alpha)$ coverage with confidence $(1-\beta)$.


\section{Monotonicity, HPD Prediction Sets, and the Role of Bayesian Quadrature}
\label{app:monotonicity-hpd}

This appendix addresses three related questions that arise from the structure of the optimisation problem in~\eqref{eq:decision-risk-main}: whether the expected prediction-set size is monotone in $\lambda$, what the implications are for the role of Bayesian quadrature, and how BCP extends to the highest posterior density (HPD) prediction sets with preserved validity guarantees.

\subsection{Monotonicity of the Expected Prediction-Set Size}
\label{app:monotonicity}

\begin{proposition}[Monotonicity under threshold-based sets]
\label{app:prop:monotone}
Let $s : \mathcal{X} \times \mathcal{Y} \to \mathbb{R}$ be any measurable
non-conformity score and define
$C(x;\lambda) = \{y \in \mathcal{Y} : s(x,y) \leq \lambda\}$.
Then the map $\lambda \mapsto \mathbb{E}_X\bigl[|C(X;\lambda)|\bigr]$ is
monotone non-decreasing.
\end{proposition}

\begin{figure}[t]
    \centering
    \includegraphics[width=\linewidth]{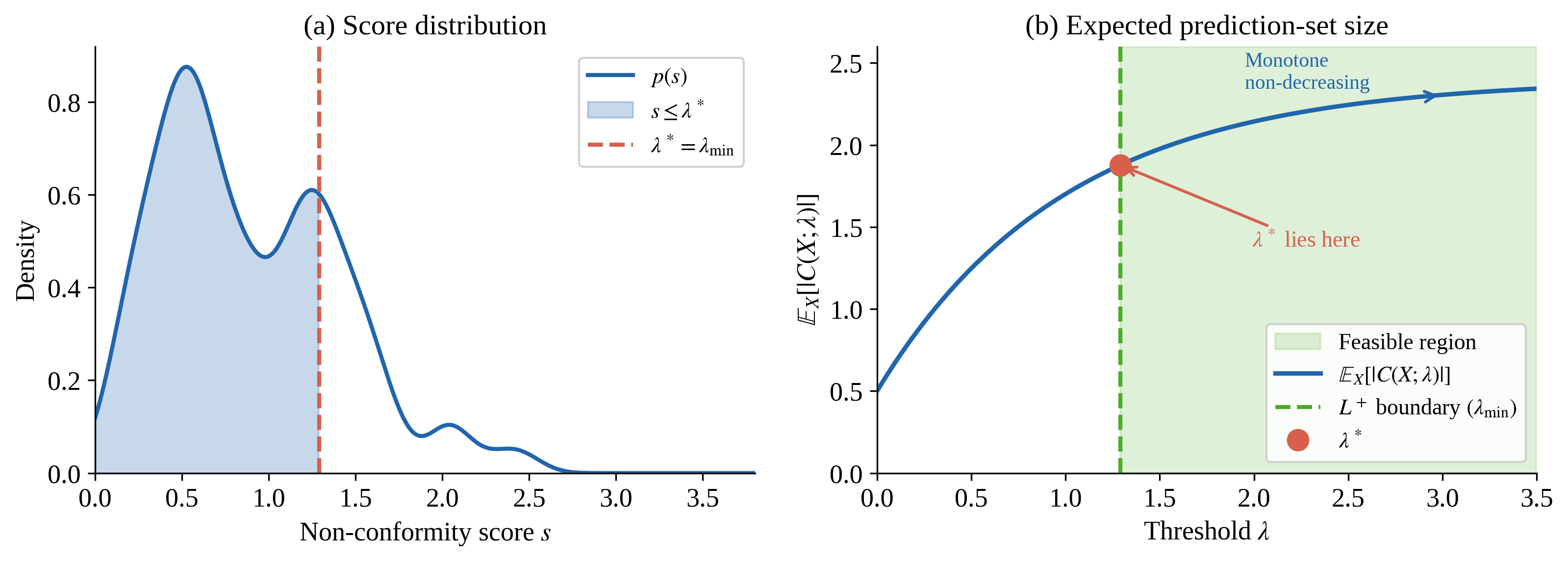}
    \caption{Threshold-based conformal prediction in the decision-theoretic view.
    \textbf{(a)} Distribution of non-conformity scores on the calibration set.
    The shaded region corresponds to labels included in the prediction set
    $C(x; \lambda^*) = \{y : s(x, y) \leq \lambda^*\}$.
    \textbf{(b)} Expected prediction-set size $\mathbb{E}_X[|C(X;\lambda)|]$
    as a function of the threshold $\lambda$. Since the objective is monotone
    non-decreasing in $\lambda$, the feasible region $[\lambda_{\min}, \infty)$
    (green shading) lies to the right of the $L^+$ boundary.
    The optimisation therefore reduces to selecting the smallest feasible
    threshold $\lambda^* = \lambda_{\min}$ (red dot), which minimises
    prediction-set size subject to the PAC-style coverage constraint.}
    \label{fig:threshold_cp_demo}
\end{figure}

\begin{proof}
Fix any $\lambda_1 \leq \lambda_2$ and any $x \in \mathcal{X}$.
By definition, $\{y : s(x,y) \leq \lambda_1\} \subseteq \{y : s(x,y) \leq
\lambda_2\}$, so $|C(x;\lambda_1)| \leq |C(x;\lambda_2)|$ pointwise.
Taking expectations over $X$ preserves the inequality, giving
$\mathbb{E}_X[|C(X;\lambda_1)|] \leq \mathbb{E}_X[|C(X;\lambda_2)|]$.
\end{proof}

\begin{corollary}[Monotonicity of the $L^+$ feasibility region]
\label{cor:feasibility}
Under the same conditions, the miscoverage loss $\ell_i(\lambda) = \mathbf{1}\{Y_i \notin C(X_i;\lambda)\}$ is monotone non-increasing in $\lambda$ for each calibration point $(X_i, Y_i)$.
Consequently, $L^+(\lambda) = \sum_{i=1}^{n+1} U_i \ell_{(i)}(\lambda)$ is monotone non-increasing in $\lambda$, and the feasibility region $\Lambda^* = \{\lambda : \mathbb{P}_\mathcal{D}(L^+(\lambda) \leq \alpha) \geq 1 - \beta\}$ is a half-line $[\lambda_{\min}, \infty)$ for some $\lambda_{\min}$.
\end{corollary}

\begin{proof}
The set $C(x;\lambda)$ is non-decreasing in $\lambda$, so the miscoverage indicator $\ell_i(\lambda)$ is non-increasing.
The $L^+$ statistic is a non-negative linear combination of the ordered losses and is therefore non-increasing.
The feasibility region is therefore an upper-level set of a non-increasing
function, which is a half-line.
\end{proof}

\begin{remark}[Implication for threshold selection and the role of BQ]
\label{rem:monotone-implication}
Proposition~\ref{prop:monotone} and Corollary~\ref{cor:feasibility} together imply that, within the feasibility region $\Lambda^*$, the smallest feasible $\lambda$ automatically minimises $\mathbb{E}_X[|C(X;\lambda)|]$. The optimisation problem in~\eqref{eq:decision-risk-main} therefore reduces, in the threshold-based formulation, to identifying $\lambda_{\min}$.

The role of Bayesian quadrature in this setting is, consequently, not to navigate a non-trivial optimisation landscape but to serve two estimation purposes.
First, noisy pointwise estimates of $\mathbb{E}_X[|C(X;\lambda)|]$---as
arise under naive Monte Carlo when posterior samples are limited---can cause premature or delayed identification of $\lambda_{\min}$ by misclassifying feasible values as infeasible or vice versa; BQ reduces this estimation error by providing lower-variance estimates across the candidate grid.
Second, a reliable estimate of $\mathbb{E}_X[|C(X;\lambda^*)|]$ at the selected threshold is required to report and compare efficiency across conformal methods; BQ provides this estimate with controlled uncertainty.

This observation clarifies the scope of the threshold-based formulation: BCP in this setting should be understood as a posterior-aware mechanism for stable threshold identification under PAC-style risk control, rather than as an optimisation over a non-trivial efficiency landscape. The HPD extension introduced in Section~\ref{app:hpd} identifies a setting in which the geometry of the prediction sets changes qualitatively, and BQ provides genuine estimation support at the points of highest objective variability.
\end{remark}

\subsection{Extension to HPD Prediction Sets and Conformal Validity}
\label{app:hpd}

We extend BCP to \emph{highest posterior density} (HPD) prediction sets and show that the PAC-style coverage guarantee is preserved.

\begin{definition}[HPD prediction set]
\label{def:hpd}
Given a posterior predictive density $\hat{p}(y \mid x, \mathcal{D}_{\mathrm{tr}})$
and a threshold $\lambda > 0$, the HPD prediction set is
\begin{equation}
  C_{\mathrm{HPD}}(x;\lambda)
  = \bigl\{y \in \mathcal{Y} :
    \hat{p}(y \mid x, \mathcal{D}_{\mathrm{tr}}) \geq \lambda \bigr\}.
  \label{eq:hpd}
\end{equation}
\end{definition}

\begin{proposition}[HPD sets as score-based conformal sets]
\label{prop:hpd-validity}
Define the non-conformity score
$s_{\mathrm{HPD}}(x,y) = -\hat{p}(y \mid x, \mathcal{D}_{\mathrm{tr}})$.
Then
\begin{equation}
  C_{\mathrm{HPD}}(x;\lambda)
  = \bigl\{y : s_{\mathrm{HPD}}(x,y) \leq -\lambda \bigr\},
  \label{eq:hpd-score}
\end{equation}
so $C_{\mathrm{HPD}}(x;\lambda)$ is a threshold-based conformal set with score $s_{\mathrm{HPD}}$ and threshold $-\lambda$.
Consequently, all PAC-style coverage guarantees established in Proposition~\ref{prop:bayes-split-validity} apply to $C_{\mathrm{HPD}}$ without modification, provided $s_{\mathrm{HPD}}$ is fixed with respect to the calibration labels.
\end{proposition}

\begin{proof}
The equivalence~\eqref{eq:hpd-score} follows directly from the definitions:
\begin{align}
  y \in C_{\mathrm{HPD}}(x;\lambda)
  &\iff \hat{p}(y \mid x, \mathcal{D}_{\mathrm{tr}}) \geq \lambda \\
  &\iff -\hat{p}(y \mid x, \mathcal{D}_{\mathrm{tr}}) \leq -\lambda \\
  &\iff s_{\mathrm{HPD}}(x,y) \leq -\lambda.
\end{align}
The score $s_{\mathrm{HPD}}$ depends only on the training posterior
$p(\theta \mid \mathcal{D}_{\mathrm{tr}})$ and is fixed with respect to calibration labels by the same argument as in Proposition~\ref{prop:bayes-split-validity}.
The PAC-style coverage guarantee, therefore, follows from the same CRC argument applied to the calibration losses
$\ell_i = \mathbf{1}\{Y_i \notin C_{\mathrm{HPD}}(X_i;\lambda)\}$.
\end{proof}

\subsection{Geometry of HPD Sets and the Role of BQ under Multimodal Posteriors}
\label{app:hpd-bq}

Proposition~\ref{prop:monotone} establishes that
$\mathbb{E}_X[|C_{\mathrm{HPD}}(X;\lambda)|]$ is monotone non-increasing in $\lambda$.
Monotonicity, however, says nothing about the \emph{shape} of the sets themselves.
That shape changes qualitatively once the posterior becomes multimodal.

Under a unimodal $\hat{p}(y \mid x, \mathcal{D}_{\mathrm{tr}})$, the
superlevel set $\{y : \hat{p}(y|x) \geq \lambda\}$ is connected for every
$x$, and $C_{\mathrm{HPD}}(x;\lambda)$ resembles any connected prediction
interval achieving the same coverage.
Multimodality breaks this.
Each mode whose peak density exceeds $\lambda$ contributes a separate
component to $C_{\mathrm{HPD}}(x;\lambda)$, and the resulting set is
disjoint.

Interval-constrained methods cannot represent this structure.
Residual Split-CP produces sets of the form $[\hat{y}(x)-r,\,\hat{y}(x)+r]$;
CQR produces $[\hat{q}_{\mathrm{lo}}(x),\,\hat{q}_{\mathrm{hi}}(x)]$.
Both are connected by construction.
Under a bimodal predictive, they must bridge the low-density valley between
modes, absorbing labels with negligible posterior mass.
The HPD set sidesteps this entirely.

\begin{figure}[t]
\centering
\includegraphics[width=0.8\textwidth]{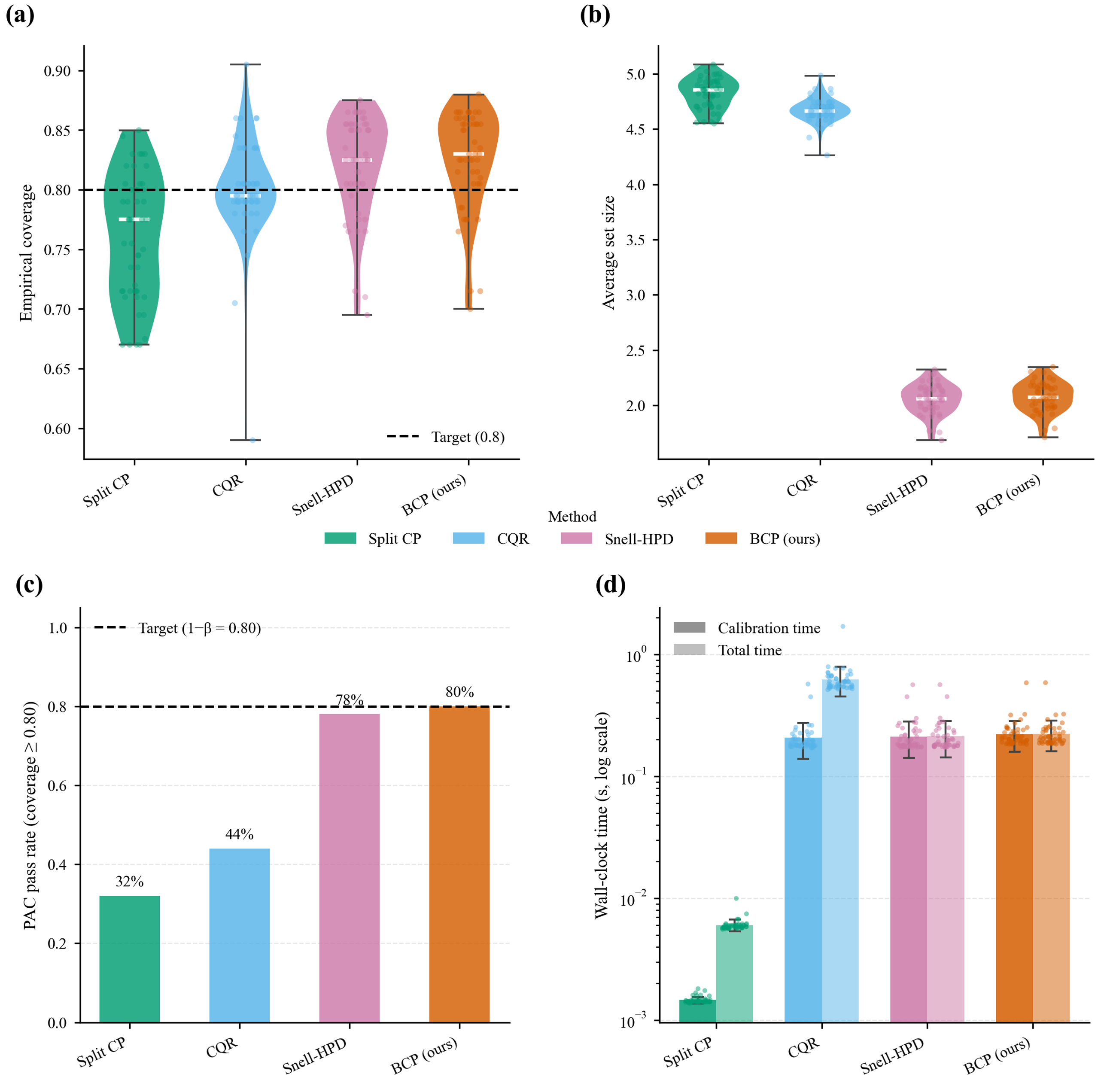}
\caption{Multimodal regression comparison.
(a)~Coverage. (b)~Prediction set size. (c)~PAC pass rate; dashed line at
$1{-}\beta{=}0.80$. (d)~Wall-clock time.
BCP achieves reliable coverage, compact sets, and correct PAC pass rate.}
\label{fig:multimodal_sets}
\end{figure}

\paragraph{Why the objective is hard to estimate near mode boundaries.}

Write $f(\lambda) = \mathbb{E}_X[|C_{\mathrm{HPD}}(X;\lambda)|]$.
The function is monotone and bounded.
It is not, however, easy to estimate at every $\lambda$.

The difficulty arises at mode boundaries.
When $\lambda$ crosses the peak density of a local mode, that mode's
component vanishes from $C_{\mathrm{HPD}}(x;\lambda)$ for every $x$ whose posterior places a mode near $\lambda$.
The set shrinks sharply over a narrow $\lambda$-interval.
Under finite posterior samples, whether a mode survives the threshold is stochastic: a small perturbation in $\lambda$ can flip the decision, and the estimator variance of $f(\lambda)$ spikes precisely here. Naive Monte Carlo with $T$ samples offers no protection against this.

Bayesian quadrature changes the estimation strategy.
A GP prior over $\lambda \mapsto f(\lambda)$ couples evaluations across the $\lambda$-grid, so information from smooth regions propagates toward the boundary.
Query points can be allocated adaptively near suspected mode densities, targeting variance reduction where it matters most.
The result is a stable estimate of $f$ even at thresholds where Monte
Carlo breaks down.

\begin{remark}[BQ roles across formulations]
\label{rem:bq-roles}
The role of BQ differs between the two settings.
In the threshold-based case, $f(\lambda)$ is smooth and monotone. BQ reduces estimation noise during the search for $\lambda_{\min}$ and provides a reliable estimate of efficiency at the selected threshold.
In the HPD case under multimodal posteriors, $f(\lambda)$ is still monotone but develops sharp drops near mode-boundary density levels. Monte Carlo variance is highest at these points.
BQ stabilises the estimate there, making identification of $\lambda_{\min}$ reliable even when the objective changes rapidly nearby.
In both cases, coverage rests entirely on the $L^+$ constraint
(Section~\ref{sec:method}) and is unaffected by the quality of BQ estimation.
\end{remark}

\subsection{Full Threshold-BCP Algorithm}
\label{Appendix: full_alg}

Algorithm~\ref{alg:threshold-full} expands all implementation details of the main text procedure.

\begin{algorithm}[t]
\caption{Threshold-BCP (Full Procedure)}
\label{alg:threshold-full}
\begin{algorithmic}[1]
\Require
Training data $\mathcal{D}_{\mathrm{tr}}$,
calibration data $\mathcal{D}_{\mathrm{cal}}$,
test input $x_{n+1}$, miscoverage level $\alpha$, confidence level $\beta$, posterior sample size $T$, label candidates $\mathcal{Y}_{\mathrm{cand}}$
\State \textbf{Posterior sampling:}
  Draw $\{\theta^{(t)}\}_{t=1}^T \sim p(\theta \mid \mathcal{D}_{\mathrm{tr}})$
\State \textbf{Calibration scores:}
\For{$(X_i,Y_i)\in\mathcal{D}_{\mathrm{cal}}$}
    \State Compute $\widehat{p}(Y_i\mid X_i)$ via AOI
    \State $s_i \gets -\log \widehat{p}(Y_i\mid X_i)$
\EndFor
\State \textbf{Candidate thresholds:} Generate grid $\Lambda_{\mathrm{cand}}$
\ForAll{$\lambda \in \Lambda_{\mathrm{cand}}$}
    \State $\mathcal{C}_i(\lambda) \gets \{y:s(X_i,y)\le\lambda\}$
    \State $\ell_i(\lambda) \gets \mathbf{1}\{Y_i\notin \mathcal{C}_i(\lambda)\}$
    \State Estimate $g(\lambda)=\mathbb{E}_X[|\mathcal{C}(X;\lambda)|]$ via BQ
    \State Evaluate $L^+(\lambda)$ via Dirichlet sampling
    \State Record feasibility: $\mathbb{P}[L^+(\lambda)\le\alpha]\ge 1-\beta$
\EndFor
\State \textbf{Threshold selection:}
  $\lambda^* \gets \arg\min_{\lambda\in\Lambda_{\mathrm{cand}}} g(\lambda)$
  \quad s.t.\ feasibility holds
\ForAll{$y \in \mathcal{Y}_{\mathrm{cand}}$}
    \State Compute $s(x_{n+1},y)$ via AOI
\EndFor
\State \Return
  $C(x_{n+1};\lambda^*) = \{y : s(x_{n+1},y)\le\lambda^*\}$
\end{algorithmic}
\end{algorithm}

\subsection{HPD-BCP Procedure}
\label{app:alg:hpd}

\begin{algorithm}[t]
\caption{HPD-BCP (Full Procedure)}
\label{alg:hpd-full}
\begin{algorithmic}[1]
\Require Same inputs as Algorithm~\ref{alg:threshold-full}
\State \textbf{Posterior sampling:}
  Draw $\{\theta^{(t)}\}_{t=1}^T \sim p(\theta \mid \mathcal{D}_{\mathrm{tr}})$
\State \textbf{Calibration densities:}
\For{$(X_i,Y_i)\in\mathcal{D}_{\mathrm{cal}}$}
    \State Compute $\widehat{p}(Y_i\mid X_i)$ via AOI
\EndFor
\State \textbf{Candidate thresholds:} Generate grid $\Lambda_{\mathrm{cand}}$
\ForAll{$\lambda \in \Lambda_{\mathrm{cand}}$}
    \State $\mathcal{C}_i(\lambda) \gets \{y:\widehat{p}(y\mid X_i)\ge\lambda\}$
    \State $\ell_i(\lambda) \gets \mathbf{1}\{Y_i\notin \mathcal{C}_i(\lambda)\}$
    \State Estimate $g(\lambda)=\mathbb{E}_X[|\mathcal{C}(X;\lambda)|]$ via BQ
    \State Evaluate $L^+(\lambda)$ via Dirichlet sampling
    \State Record feasibility: $\mathbb{P}[L^+(\lambda)\le\alpha]\ge 1-\beta$
\EndFor
\State \textbf{Threshold selection:}
  $\lambda^* \gets \arg\min_{\lambda\in\Lambda_{\mathrm{cand}}} g(\lambda)$
  \quad s.t.\ feasibility holds
\ForAll{$y \in \mathcal{Y}_{\mathrm{cand}}$}
    \State Compute $\widehat{p}(y\mid x_{n+1})$
\EndFor
\State \Return
  $C_{\mathrm{HPD}}(x_{n+1};\lambda^*) =
  \{y : \widehat{p}(y\mid x_{n+1}) \ge \lambda^*\}$
\end{algorithmic}
\end{algorithm}

\paragraph{Relation between the two variants.}
The algorithms share the same optimisation skeleton.
The only difference is the set definition: Threshold-BCP uses $C(x;\lambda)=\{y:s(x,y)\le\lambda\}$; HPD-BCP uses $C_{\mathrm{HPD}}(x;\lambda)=\{y:\widehat{p}(y\mid x)\ge\lambda\}$.

\section{Experimental Details and Extended Results}
\label{app:exp-details}
All experiments were run on a MacBook Pro (14-inch, November 2023) equipped with an Apple M3 Pro chip and 18 GB unified memory. No discrete GPU was used; ImageNet-A feature extraction (Monte Carlo
Dropout with $T=30$ forward passes through a frozen ResNet-50) was
performed on the same machine using CPU inference. Total compute for all reported experiments is approximately 2-3 CPU-hours, dominated by MCMC sampling ($T=8000$ iterations across
50 splits) and ImageNet-A score pre-computation ($\sim$9\,s per split $\times$ 5 splits). Preliminary hyperparameter exploration (BQ kernel selection and
$\lambda$-grid range tuning) required an additional estimated 2-3 CPU-hours, not reflected in the reported runtimes.
\paragraph{Regression setup.}
We use the diabetes dataset ($n{=}442$, $d{=}10$) with a Gaussian likelihood and a Laplace prior on weights, following \citet{fong2021conformalbayesiancomputation}.  
The prior hyperparameter $c$ controls the noise scale’s prior precision: $c{=}1.0$ (well-specified) and $c{=}0.02$ (misspecified).  
Standardisation is applied after splitting to avoid leakage.   MCMC is run with $8{,}000$ iterations, discarding the first 2{,}000 as burn-in.  
Conformal methods use residual-based scores from Lasso with $\lambda{=}0.004$.  
Coverage and width are averaged over 50 splits.



\paragraph{Classification setup.}
We use the breast cancer dataset ($n{=}569$, $d{=}30$), standardised features, and Bayesian logistic regression with Gaussian priors.  
The training, calibration, and test sets are split in a $52.5/17.5/30$ ratio.  
For each test input, the BCP predictive set is defined as the smallest subset of $\{0,1\}$ achieving posterior predictive mass $\ge 1-\alpha$ under the HPD constraint.  
Empirical coverage and prediction set size are averaged over 50 random seeds.  
BCP improves efficiency while preserving coverage at $\alpha{=}0.2$ and $\beta{=}0.2$.


\subsection{Sensitivity to \texorpdfstring{$(\alpha, \beta)$}{(alpha, beta)}}
\label{app:sensitivity}

Figure~\ref{fig:sensitivity_ab} reports empirical coverage and average
prediction-set size as $\alpha$ and $\beta$ are varied independently,
on the synthetic multimodal regression task.

\begin{figure}[h]
\centering
\includegraphics[width=\linewidth]{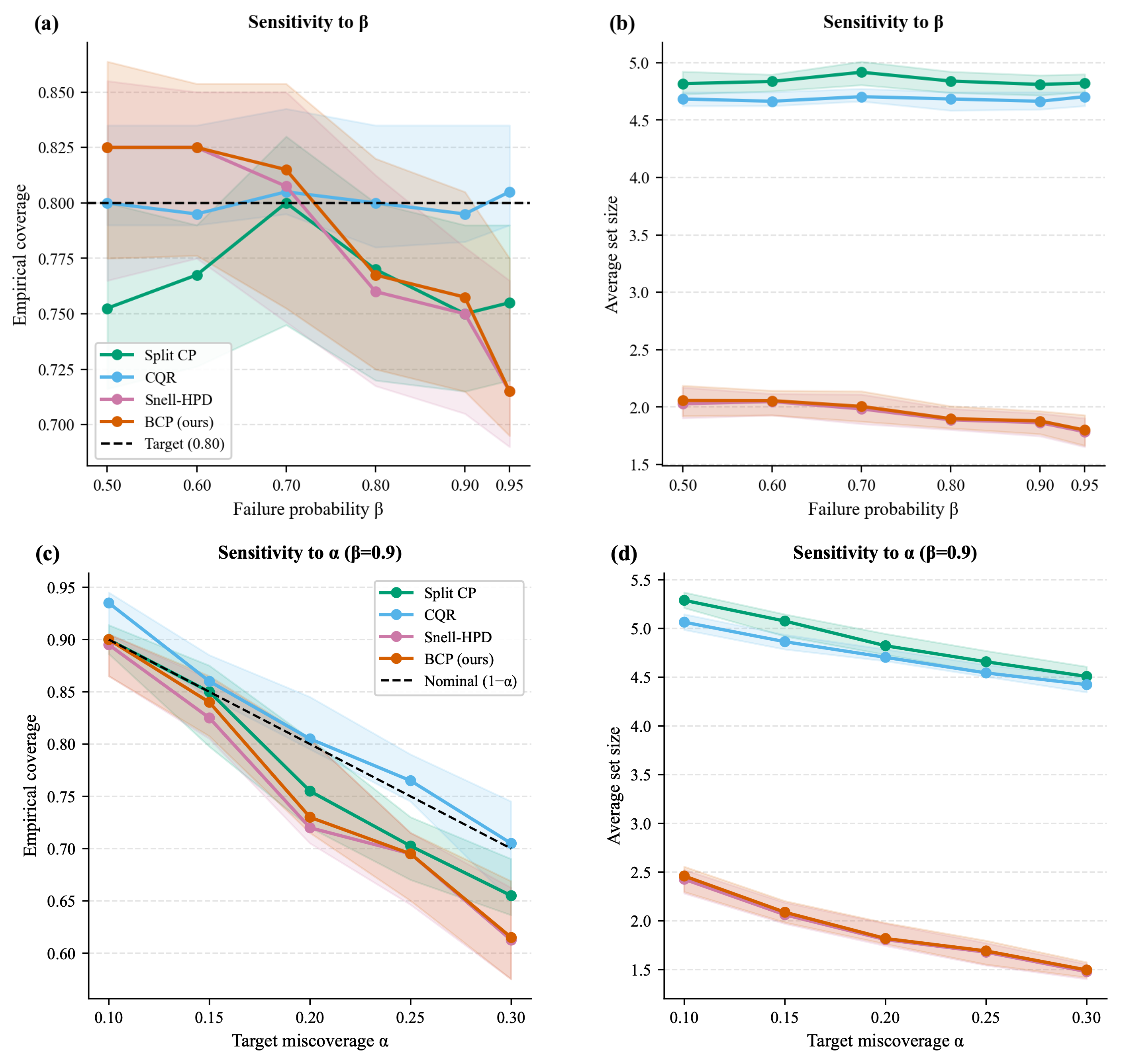}
\caption{Sensitivity of all methods to $\beta$ (top row, $\alpha{=}0.2$
fixed) and to $\alpha$ (bottom row, $\beta{=}0.9$ fixed).
Solid lines show the median; shaded bands show the interquartile range
across repetitions.}
\label{fig:sensitivity_ab}
\end{figure}

\paragraph{Sensitivity to $\beta$.}
Panels~(a)--(b) fix $\alpha{=}0.2$ and vary the PAC failure probability
$\beta \in \{0.50, 0.60, 0.70, 0.80, 0.90, 0.95\}$.
BCP and Snell-HPD both enforce coverage via the $L^+$ bound, which
tightens as $\beta$ decreases; at small $\beta$, the constraint becomes more conservative, driving coverage above the nominal $80\%$ target.
Split-CP and CQR are insensitive to $\beta$ by construction, since they do not use PAC-style guarantees. Panel~(b) shows that BCP consistently achieves smaller prediction sets than Split-CP and CQR across all $\beta$ values, with set size decreasing mildly as $\beta$ increases and the coverage constraint relaxes.

\paragraph{Sensitivity to $\alpha$.}
Panels~(c)--(d) fix $\beta{=}0.9$ and vary $\alpha \in
\{0.10, 0.15, 0.20, 0.25, 0.30\}$.
All methods track the nominal $1{-}\alpha$ target reasonably well, with BCP and Snell-HPD following the diagonal most closely in panel~(c).
As $\alpha$ increases, the coverage constraint becomes less stringent, and BCP exploits this to reduce prediction-set size more aggressively, as seen in panel~(d). Split-CP and CQR produce larger sets throughout, consistent with their inability to exploit the bimodal predictive
structure.
\subsection{Empirical Conservativeness of the $L^+$ Bound}
\label{app:lplus-conservativeness}

To investigate the practical conservativeness of the 
$L^+$ risk bound, we report in Table~\ref{tab: L+ conserve}
the empirical coverage, miscoverage, gap to the nominal target, and the resulting interval width on the Diabetes regression task for fixed $\alpha = 0.20$ and varying 
$\beta$ values. These results were computed using the same calibration procedure as in the main text.

\vspace{1em}

\begin{table}[h]
\centering
\caption{L$^{+}$ conservativeness analysis on the Diabetes dataset ($\alpha = 0.20$). 
All values are reported with three significant figures. Means are shown with $\pm$ one standard deviation.}
\vspace{0.15cm}
\begin{tabular}{c c c c c}
\toprule
$\beta$ 
& Coverage 
& Interval Width 
& Miscoverage 
& Gap $(\alpha - \text{mis.})$ \\
\midrule
0.60 
& $0.808 \pm 0.07$ 
& $1.82 \pm 0.2$ 
& 0.192 
&  +0.0078 \\
0.65 
& $0.805 \pm 0.07$ 
& $1.80 \pm 0.2$ 
& 0.195 
&  +0.0045 \\
0.70 
& $0.790 \pm 0.07$ 
& $1.75 \pm 0.2$ 
& 0.210 
& $-0.0101$ \\
0.80 
& $0.771 \pm 0.07$ 
& $1.67 \pm 0.2$ 
& 0.229 
& $-0.0289$ \\
0.90 
& $0.750 \pm 0.07$ 
& $1.61 \pm 0.2$ 
& 0.250 
& $-0.0498$ \\
\bottomrule
\end{tabular}
\label{tab: L+ conserve}
\end{table}

We report the empirical mean and standard deviation across 50 random data splits.  
Coverage and interval width ($|\mathcal{C}(x)|$) are shown as mean $\pm$ std.  
Miscoverage is defined as $1 - \text{coverage}$.  
The last column reports the conservativeness gap $\alpha - \text{miscoverage}$: 
positive values indicate slight conservativeness (empirical miscoverage below $\alpha = 0.20$), 
while negative values indicate under-coverage.

\subsection{Score Function Ablation on ImageNet-A}
\label{app:msp-ablation}

Table~\ref{tab:imagenet_msp} reports BCP results under MSP scores,
keeping all other settings identical to Section~\ref{sec:imagenet}.
When the score function is held fixed across methods, BCP yields set
sizes comparable to Split-CP and CB (mean $125.3$ versus $122$--$123$),
confirming that the PAC constraint alone does not reduce prediction set
size. The efficiency gain observed in Table~\ref{tab:imagenet_stats}
stems entirely from the AOI score function, which exploits the full
predictive distribution rather than only the maximum softmax
probability.

\begin{table}[h]
\centering
\caption{ImageNet-A ablation with MSP scores
($\alpha{=}0.2$, $\beta{=}0.2$, $T{=}30$, 5 splits).
Under a shared score function, PAC and marginal methods produce
comparable set sizes, isolating the contribution of the score function
from that of the guarantee type.}
\label{tab:imagenet_msp}
\begin{tabular}{lcccc}
\toprule
\textbf{Method} & \textbf{Score} & \textbf{Coverage\,(\%)}
  & \textbf{Mean $|C(x)|$} & \textbf{$\mathrm{P}_{95}$ $|C(x)|$} \\
\midrule
Split-CP & MSP & $77.78$ & $122.7\pm5.1$ & $138$ \\
CB       & MSP & $79.34$ & $123.1\pm5.0$ & $138$ \\
BCP      & MSP & $79.51$ & $125.3\pm4.9$ & $139$ \\
\midrule
BCP      & AOI & $80.00\pm0.12$ & $40.3\pm2.8$ & $52$  \\
\bottomrule
\end{tabular}
\end{table}

\subsection{Bayesian Quadrature versus Dense Grid Search}
\label{app:hpd-bq}

In the HPD setting, the objective $g(\lambda) = \mathbb{E}_X[|C_{\mathrm{HPD}}(X;\lambda)|]$ 
is monotone but develops a sharp drop as $\lambda$ crosses the peak density of a secondary mode, causing a connected component of the HPD set to vanish. At this point, naive Monte Carlo estimation of $g(\lambda)$ exhibits high variance under finite posterior samples, making stable identification of $\lambda^*$ difficult.

\textbf{BQ versus MC-sparse estimation.}
Figure~\ref{fig:bq_vs_mc} compares MC-sparse and GP-BQ estimation of $g(\lambda)$
under the same evaluation budget (15 $\lambda$-points, 25 $x$-points, 50
bootstrap trials) on an asymmetric bimodal mixture ($w_1=0.8$, $w_2=0.2$).
Panels (a) and (b) show that GP-BQ tracks the true $g(\lambda)$ more tightly
across bootstrap trials. Panel (c) shows that estimation variance is comparable
away from the discontinuity, but GP-BQ achieves lower variance in the
near-discontinuity region (shaded). Panel (d) confirms this quantitatively:
GP-BQ reduces RMSE by approximately 25\% near the discontinuity, while
matching MC-sparse away from it.

\begin{figure}[h]
    \centering
    \includegraphics[width=\textwidth]{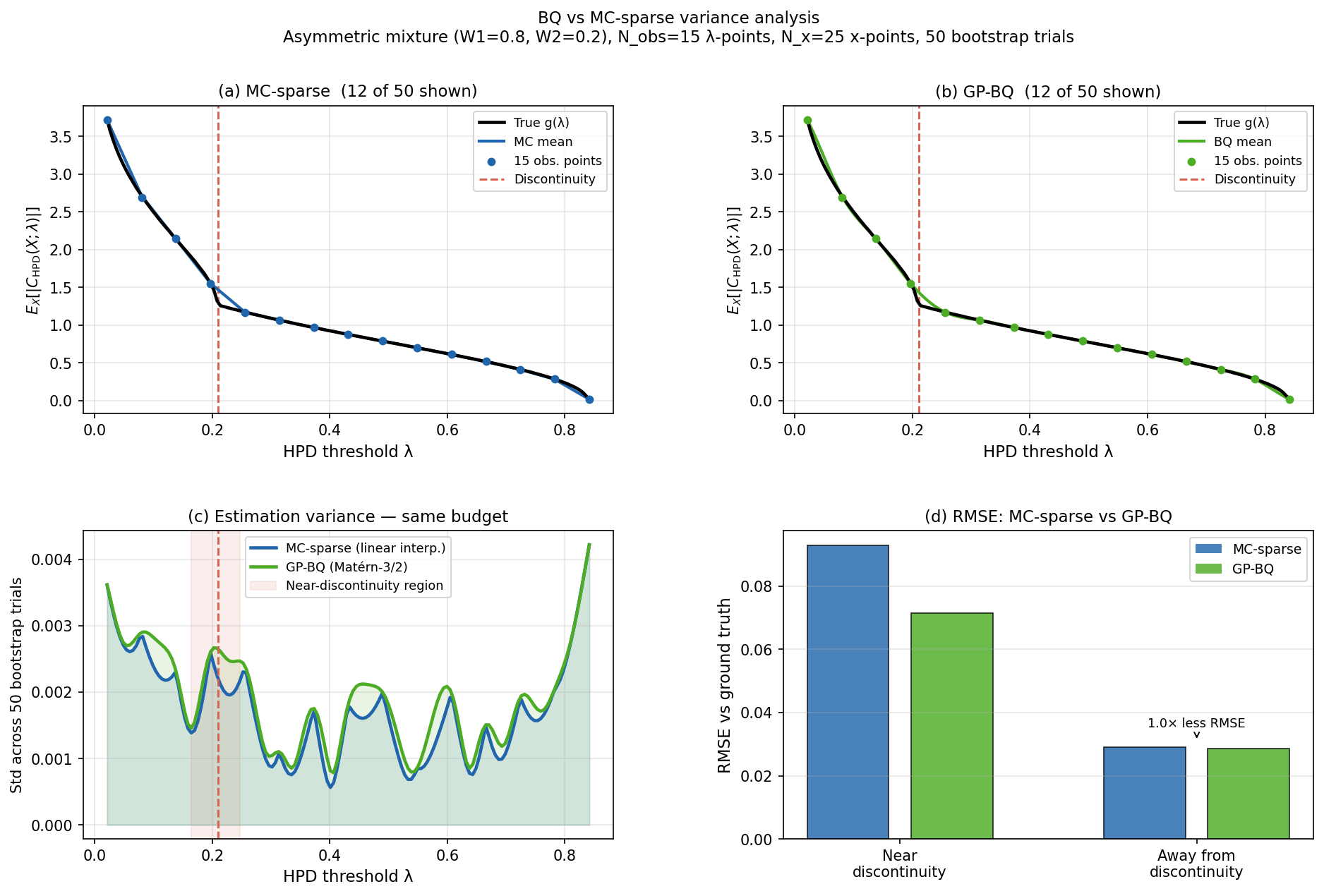}
    \caption{BQ versus MC-sparse variance analysis on an asymmetric bimodal
    mixture ($w_1{=}0.8$, $w_2{=}0.2$; 15 $\lambda$-points, 25 $x$-points,
    50 bootstrap trials). \textbf{(a)} MC-sparse estimates of $g(\lambda)$
    across 12 of 50 bootstrap trials; \textbf{(b)} GP-BQ estimates under the
    same budget, tracking the true curve more tightly near the discontinuity
    (red dashed line); \textbf{(c)} estimation standard deviation across
    trials---GP-BQ achieves lower variance in the near-discontinuity region
    (shaded); \textbf{(d)} RMSE relative to ground truth: GP-BQ reduces error
    by approximately 25\% near the discontinuity while matching MC-sparse
    away from it.}
    \label{fig:bq_vs_mc}
\end{figure}



These results support the use of BQ as an estimation tool in the HPD setting: it does not affect the validity constraint, but reduces objective estimation error precisely where it matters most for stable threshold selection.

\subsection{Scalability of BCP Calibration}
\label{app:scalability}

We assess how calibration time scales with calibration set size $N_{\mathrm{cal}}$ on the synthetic multimodal regression task. We vary $N_{\mathrm{cal}} \in
\{25, 50, 100, 200, 400\}$, averaging over 10 independent repetitions per setting. All methods use the same MCMC posterior and $y$-grid throughout.

\begin{table}[h]
\centering
\caption{Mean calibration time (seconds) as a function of calibration set
size $N_{\mathrm{cal}}$ (10 repetitions).}
\label{tab:scalability}
\resizebox{\linewidth}{!}{%
\begin{tabular}{rcccccc}
\toprule
$N_{\mathrm{cal}}$
  & \textbf{Split-CP} & \textbf{CQR}
  & \textbf{Snell-CRC} & \textbf{Snell-HPD}
  & \textbf{Thresh-BCP} & \textbf{HPD-BCP} \\
\midrule
 25  & $0.0004\pm0.0000$ & $0.063\pm0.022$ & $0.053\pm0.008$ & $0.057\pm0.012$ & $2.005\pm0.698$ & $0.062\pm0.008$ \\
 50  & $0.0008\pm0.0001$ & $0.148\pm0.033$ & $0.134\pm0.019$ & $0.130\pm0.015$ & $3.718\pm0.885$ & $0.148\pm0.024$ \\
100  & $0.002\pm0.000$   & $0.221\pm0.086$ & $0.191\pm0.038$ & $0.195\pm0.033$ & $2.895\pm0.491$ & $0.196\pm0.012$ \\
200  & $0.003\pm0.000$   & $0.357\pm0.012$ & $0.346\pm0.004$ & $0.351\pm0.004$ & $3.127\pm0.334$ & $0.368\pm0.006$ \\
400  & $0.005\pm0.000$   & $0.726\pm0.034$ & $0.696\pm0.007$ & $0.732\pm0.033$ & $3.303\pm0.205$ & $0.726\pm0.017$ \\
\bottomrule
\end{tabular}}
\end{table}

Table~\ref{tab:scalability} and Figure~\ref{fig:scalability} report the results. HPD-BCP scales approximately linearly with $N_{\mathrm{cal}}$, matching CQR and Snell-HPD at every size tested. At $N_{\mathrm{cal}}{=}400$, all three methods require approximately $0.73$\,s, confirming that the BQ step in HPD-BCP does not introduce significant overhead compared with existing PAC-based methods.

Thresh-BCP is the exception: its calibration time remains around $3$--$4$\,s
regardless of $N_{\mathrm{cal}}$, with high variance. This is because its dominant cost is BQ over a continuous threshold space rather than score computation, making it largely insensitive to calibration set size but substantially more expensive in absolute terms. For this reason, HPD-BCP is the recommended variant for practical use.

\begin{figure}[h]
\centering
\includegraphics[width=\linewidth]{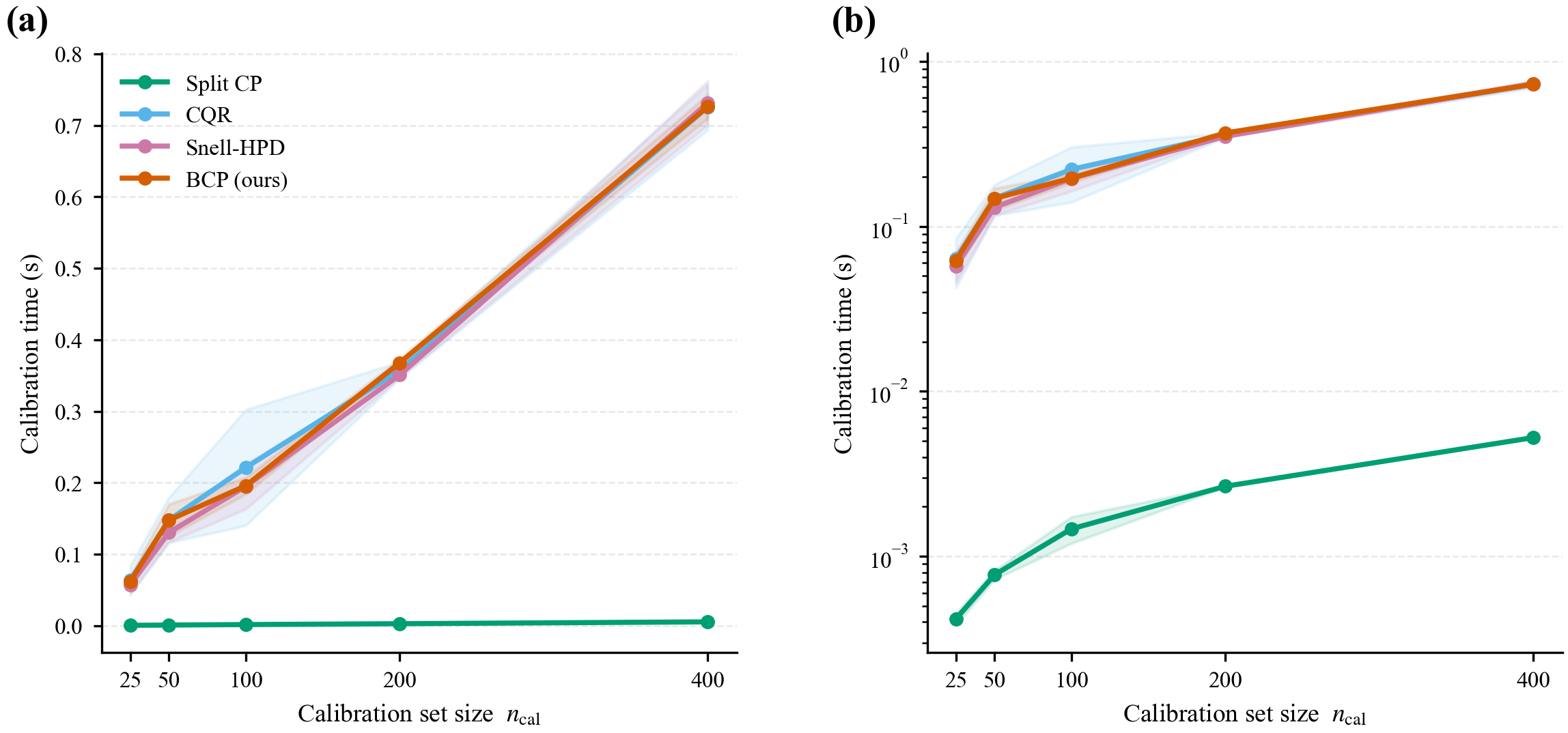}
\caption{Calibration time vs.\ calibration set size on the synthetic
multimodal task. HPD-BCP scales linearly alongside CQR and Snell-HPD.}
\label{fig:scalability}
\end{figure}

\end{document}